\documentclass{article}

\usepackage{microtype}
\usepackage{graphicx}
\usepackage{subfigure}
\usepackage{booktabs} 

\usepackage{hyperref}

\usepackage[accepted]{icml2019} 

\icmltitlerunning{Trust Region Value Optimization using Kalman Filtering} 

\usepackage{times}

\usepackage{pifont}

\usepackage{natbib}

\usepackage{algorithm}
\usepackage{algorithmic}
\usepackage{multirow}
\usepackage{cancel}
\usepackage{empheq}

\renewcommand{\algorithmicrequire}{\textbf{Input:}}
\renewcommand{\algorithmicensure}{\textbf{Output:}}

\usepackage{bm}
\usepackage{amssymb,amsmath,amsfonts,graphicx,amsthm}
\usepackage{chngcntr}
\usepackage{apptools}
\newtheorem{theorem}{Theorem}
\newtheorem{assumption}{Assumption}
\newtheorem{corollary}{Corollary}
\newtheorem{lemma}{Lemma}
\AtAppendix{\counterwithin{theorem}{section}}
\AtAppendix{\counterwithin{assumption}{section}}
\AtAppendix{\counterwithin{corollary}{section}}
\AtAppendix{\counterwithin{lemma}{section}}
\begin{document} 
	\twocolumn[
	\icmltitle{Trust Region Value Optimization using Kalman Filtering} 
	
	
	
	\icmlsetsymbol{equal}{*}
	
	\begin{icmlauthorlist}
		\icmlauthor{Shirli Di-Castro Shashua}{tech}
		\icmlauthor{Shie Mannor}{tech}
	\end{icmlauthorlist}
	
	\icmlaffiliation{tech}{Technion, Israel}

	\icmlcorrespondingauthor{Shirli Di-Castro Shashua}{shirlidi@tx.technion.ac.il}
	\icmlcorrespondingauthor{Shie Mannor}{shie@ee.technion.ac.il}
	
	\icmlkeywords{Machine Learning, ICML, Kalman}
	
	\vskip 0.3in
	]
	
	
	
	\printAffiliationsAndNotice{}  

\begin{abstract}
	Policy evaluation is a key process in reinforcement learning. It assesses a given policy using estimation of the corresponding value function. When using a parameterized function to approximate the value, it is common to optimize the set of parameters by minimizing the sum of squared Bellman Temporal Differences errors. However, this approach ignores certain distributional properties of both the errors and value parameters. Taking these distributions into account in the optimization process can provide useful information on the amount of confidence in value estimation. In this work we propose to optimize the value by minimizing a regularized objective function which forms a trust region over its parameters. We present a novel optimization method, the Kalman Optimization for Value Approximation (KOVA), based on the Extended Kalman Filter. KOVA minimizes the regularized objective function by adopting a Bayesian perspective over both the value	parameters and noisy observed returns. This distributional property provides information on parameter uncertainty in addition to value estimates. We provide theoretical results of our approach and analyze the performance of our proposed optimizer on domains with large state and action spaces. 
\end{abstract} 

\section{Introduction}
\label{introduction}
Reinforcement learning (RL) solves sequential decision making problems by considering an agent that interacts with the environment and seeks for the optimal policy \cite{sutton1998reinforcement}. During the learning process, the agent is required to evaluate its policies using a value function. In many real world RL domains, such as robotics, games and autonomous driving cars, the state and action spaces are large, hence the value function is approximated, e.g., using a Deep Neural Network (DNN). A common approach is to optimize a set of parameters by minimizing the sum of squared Bellman Temporal Differences (TD) errors \cite{dann2014policy}. There are two underlying assumptions in this approach: first, the value and its parameters are deterministic; second, the Bellman TD errors are independent Gaussian random variables (RVs) with zero mean and a fixed variance. Although being a commonly used objective function, these underlying assumptions may not be suitable for the policy evaluation task in RL. Distributional RL \cite{bellemare2017distributional} refers to the second assumption and argues in favor of a full distribution perspective over the sum of discounted rewards for a fixed policy. In particular, learning this distribution is meaningful in presence of value approximation. However, in their formulation the value parameters are still considered deterministic and they do not provide an amount of confidence for the value estimates. 

Treating the value or its parameters as RVs has been investigated in the RL literature. \citet{engel2003bayes,engel2005reinforcement} used Gaussian Processes (GP) for the value and the return to capture uncertainties in policy evaluation. \citet{geist2010kalman} proposed to use the Unscented Kalman filter (UKF) to learn the uncertainty in value parameters. Their formulation requires many samples of parameters in each training step, which is not feasible in Deep Reinforcement Learning (DRL) with large state and action spaces.

Motivated by the works of \citet{engel2003bayes,engel2005reinforcement} and \citet{geist2010kalman}, we present in this work a unified framework for addressing uncertainties while approximating the value in DRL domains. Our framework incorporates the well-known Kalman filter estimation techniques with RL principles to improve value approximation. The Kalman filter \cite{kalman1960new} and its variant for nonlinear approximations, the Extended Kalman filter (EKF) \cite{anderson1979optimal, gelb1974applied}, are used for on-line tracking and for estimating states in dynamic environments through indirect noisy observations. These methods have been successfully applied to numerous control dynamic systems such as navigation and tracking targets \cite{sarkka2013bayesian}. The Kalman filter can also be used for parameter estimation in approximation functions, where parameters replace the states of dynamic systems. \newline 

We develop a new optimization method for policy evaluation based on the EKF formulation. Figure \ref{fig:kalman_distribution_diagram} illustrates our Bayesian perspective over value parameters and noisy observed returns. Our proposed method has the following properties: It forms a trust region over the value parameters, based on their uncertainty covariance; It is aimed at tracking the solution rather than converging to it; It incrementally updates the parameters and the error covariance matrix, hence avoids sampling the parameters as is often required in Bayesian methods; It adjusts suitable learning rate to each individual parameter through the Kalman gain, thus the learning procedure does not depend on the parameterization of the value. 

Our main contributions are: (1) Developing a new regularized objective function for approximating values in the policy evaluation task. The regularization term accounts for both parameters and observations uncertainties. (2) Presenting a novel optimization algorithm, {\bf K}alman {\bf O}ptimization for {\bf V}alue {\bf A}pproximation (KOVA), and prove that it minimizes at each time step the regularized objective function. This optimizer can be easily plugged into any policy optimization algorithm and improve it. (3) Beyond RL context, we present the connection between EKF and the incremental Gauss-Newton method, the on-line natural gradient and the Kullback Leibler (KL) divergence, and explain how our objective function forms a trust region over the value parameters. (4) Demonstrating the improvement achieved by our optimizer on several control tasks with large state and action spaces. 

 \begin{figure}[t]
 	\centering{
 		\includegraphics[width=1.0\linewidth,height=0.5\textheight,keepaspectratio]{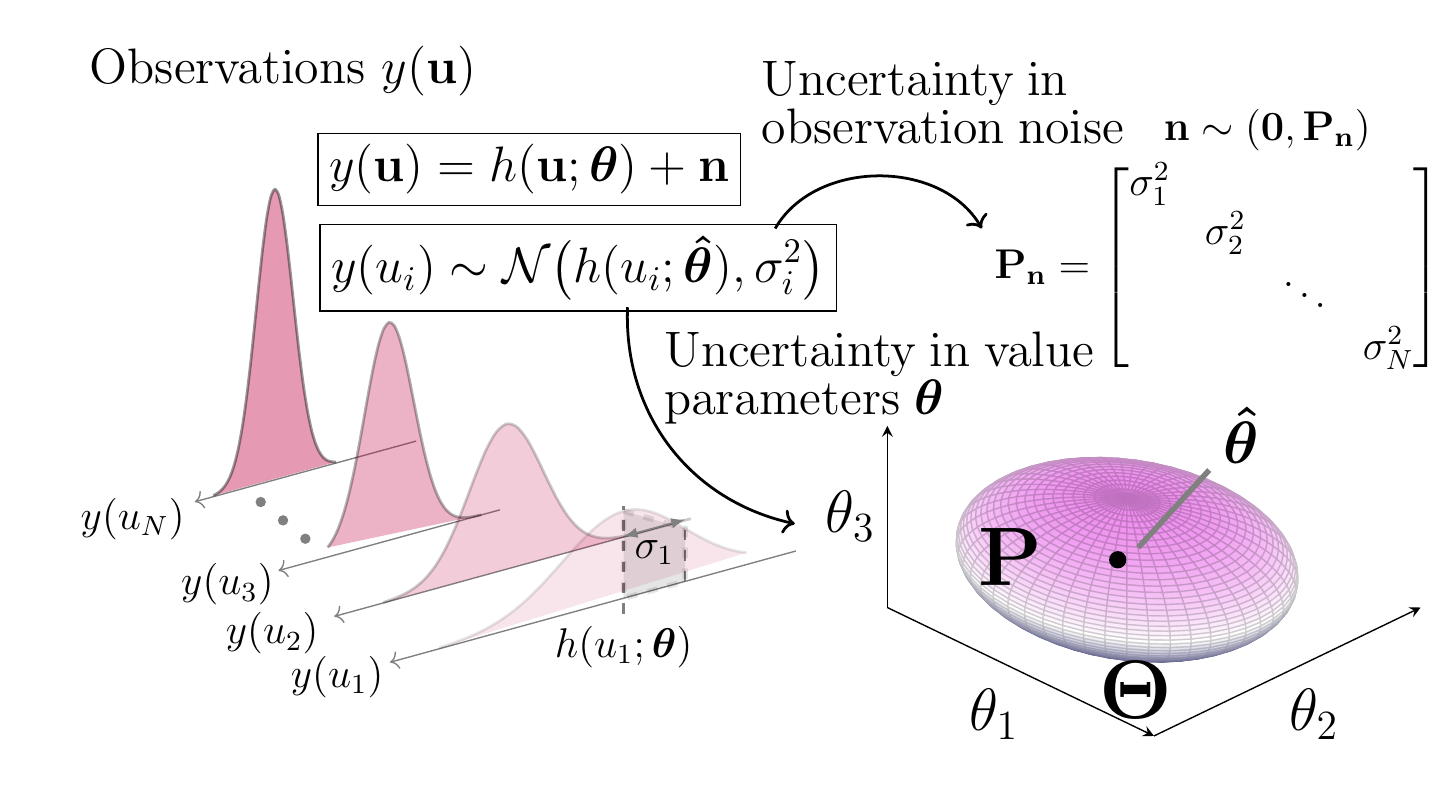}}
 	\caption{Illustration of our proposed model: a Bayesian perspective for the policy evaluation problem in RL. The noisy observation $y(u)$ for an input $u$ (for example $u$ is a state or a state-action pair and $y(u)$ is the a sum of discounted n-step rewards from this state) is decomposed into its mean, the value $h(u; \boldsymbol{\theta})$ and a random zero-mean noise $n$. The randomness of $y(u)$ originates from two sources: (i) the random noise $n$ which relates to the stochasticity of the transitions in the trajectory and to the possibly random policy. (ii) the randomness of $h$ through its dependency on the random parameters $\boldsymbol{\theta}$. In the context of RL, this randomness can be related to uncertainty regarding the MDP model that generated the noisy observations. }
 	\label{fig:kalman_distribution_diagram}
 \end{figure}

\section{Background}
\subsection{Reinforcement Learning and MDPs}
\label{Reinforcement learning}
The standard RL setting considers an interaction of an agent with an environment $\mathcal{E}$ for a discrete number of time steps. The environment is modeled as a Markov Decision Process (MDP) $\{\mathcal{S}, \mathcal{A}, P, R, \gamma\}$ where $\mathcal{S}$ is a finite set of states, $\mathcal{A}$ is a finite set of actions, $P: \mathcal{S} \times \mathcal{A} \times \mathcal{S}\rightarrow [0,1]$ is the state transition probabilities for each state $s$ and action $a$, $R: \mathcal{S} \times \mathcal{A} \rightarrow \mathbb{R}$ is a deterministic and bounded reward function and $\gamma$ is a discount factor. At each time step $t$, the agent observes state $s_t \in \mathcal{S}$ and chooses action $a_t \in \mathcal{A}$ according to a policy $\pi: \mathcal{S} \times \mathcal{A} \rightarrow [0,1]$.  The agent receives an immediate reward $r_t(s_t, a_t)$ and the environment stochastically steps to state $s_{t+1} \in \mathcal{S}$ according to the probability distribution  $P(s_{t+1}|s_t, a_t)$. The state value function and the state-action Q-function are used for evaluating the performance of a fixed policy $\pi$ \cite{sutton1998reinforcement}: $V^{\pi}(s)  = \mathbb{E}^{\pi}\big[ \sum_{t=0}^{\infty} \gamma^t r(s_t, a_t) | s_0 = s \big]$ and $Q^{\pi}(s, a) = \mathbb{E}^{\pi}\big[ \sum_{t=0}^{\infty} \gamma^t r_t(s_t, a_t)  | s_0 = s, a_0 = a \big]$, where $\mathbb{E}^{\pi}$ denotes the expectation with respect to the state (state-action) distribution induced by transition law $P$ and policy $\pi$.

\subsection{Value Function Estimation}
Policy evaluation, or value estimation, is a core element in RL algorithms. We will use the term {\it value function} (VF) to address the following functions: the state value function $V^{\pi}(s)$, the state-action Q-function $Q^{\pi}(s, a)
$ and the advantage function $A^{\pi}(s,a) = Q^{\pi}(s, a) - V^{\pi}(s)$. When the state or action space is large, a common approach is to approximate the VF using a parameterized function, $h(\cdot; \boldsymbol{\theta})$. We focus on general, possibly non-linear approximation functions such as DNNs that can learn effectively complex approximations.
    
\begin{table*}[t]
	\small
	\centering
	\caption{Different examples for policy optimization algorithms and their Bellman TD error  $\delta(u; \boldsymbol{\theta}_t)$ type. The decomposition of $\delta(u; \boldsymbol{\theta}_t)$ into the observation function $h(u; \boldsymbol{\theta}_t)$ and the target label $y(u)$ in the EKF model (\ref{eq:Extended-Kalman}) enables the integration of our KOVA optimizer with any policy optimization algorithm. $\boldsymbol{\theta}'$ refers to the previous network or to a target network, different than the one being trained $\boldsymbol{\theta}_t$.}
	\begin{tabular}{|l|p{40mm}|p{30	mm}|p{13mm}|p{50mm}|}
		\hline
		{\it Algorithm type} & {\it Example} & $\delta(u; \boldsymbol{\theta}_t)$ {\it type} & {\bf $h(u; \boldsymbol{\theta}_t)$ }& $y(u)$ \\ \hline \hline
		Actor-critic & A3C  \cite{mnih2016asynchronous}   &  $k$-step V-evaluation & $V(s_m; \boldsymbol{\theta}_t)$  & $\sum_{i=0}^{k-1}\gamma^i r_{m+i} + \gamma^k  V(s_{m + k}; \boldsymbol{\theta}')$ \\ \hline
		Actor-critic &  DDPG \cite{lillicrap2015continuous} & $1$-step Q-evaluation & $Q(s, a; \boldsymbol{\theta}_t)$ & $r + \gamma Q(s', \pi(s'); \boldsymbol{\theta}')$ \\ \hline
		Policy gradient & PPO \cite{schulman2017proximal} \newline TRPO \cite{schulman2015trust} & GAE \cite{schulman2015high} & $V(s_m; \boldsymbol{\theta}_t)$ & 
		$\sum_{i=0}^{\infty}(\gamma \lambda)^i \big(r_{m+i} + \gamma  V(s_{m + i + 1}; \boldsymbol{\theta}') $ \newline $ - V(s_{m + i}; \boldsymbol{\theta}') \big)  + V(s_{m}; \boldsymbol{\theta}')$ \\ \hline
		$\epsilon$-greedy & DQN \cite{mnih2013playing} & Optimality equation & $Q(s, a; \boldsymbol{\theta}_t)$ & $r + \gamma \max_{a'} Q(s', a'; \boldsymbol{\theta}')$ \\ \hline
	\end{tabular}
	\label{vf_table}
\end{table*}

A common approach for optimizing the VF parameters is to minimize at each time step $t$ the empirical mean of the squared {\it Bellman TD error}  $\delta (u; \boldsymbol{\theta}_t) \triangleq y (u) - h(u; \boldsymbol{\theta}_t)$, over a batch of $N$ samples generated form the environment $\mathcal{E}$ under a given policy:
\begin{equation}
\label{eq:vf objective}
L_t^{\text{MLE}}(\boldsymbol{\theta}_t) = \frac{1}{2N} \sum_{i=1}^{N} \delta^2 (u_i; \boldsymbol{\theta}_t).
\end{equation} 
We use the general notation $u$  to specify the {\it input} for the {\it target label} $y(u)$ and for the approximated value at time $t$, $h(u; \boldsymbol{\theta}_t)$. For example, for $h(u; \boldsymbol{\theta}_t) = V(s_m; \boldsymbol{\theta}_t)$, $u=s_m$ is the state at a discrete time $m$; For $h(u; \boldsymbol{\theta}_t) = Q(s_m, a_m; \boldsymbol{\theta}_t)$, $u=(s_m, a_m)$ is the state-action pair. In Table \ref{vf_table} we provide examples of several options for $y(u)$ and $h(u; \boldsymbol{\theta}_t)$ which clarify how this general notation can be utilized in known policy optimization algorithms. 

Traditionally, the VF is trained by stochastic gradient descent methods, estimating the loss on each experience as it is encountered, yielding the update:\\
$\boldsymbol{\theta}_{t+1}  \leftarrow  \boldsymbol{\theta}_t + \alpha \mathbb{E}_{u \sim p(\cdot)} \big[ \big(y(u) - h(u; \boldsymbol{\theta}_t)\big) \nabla_{\boldsymbol{\theta}_t} h(u; \boldsymbol{\theta}_t) \big]$, where $\alpha$ is the learning rate and $p(\cdot)$ is the experience distribution. Typically, the training procedure seeks for a point estimate of the model parameters. We will show (Section \ref{Section:EKF_RL}) that the underlying assumption on $L_t^{\text{MLE}}$ (\ref{eq:vf objective}) is that the parameters $\boldsymbol{\theta}_t$ are deterministic and that the target labels $y(u)$ are independent Gaussian RVs with mean $h(u; \boldsymbol{\theta_t})$ and a fixed variance. In Section \ref{Section:EKF} we present the EKF approach which generalizes the process of generating observations and adds flexibility to the model assumptions: the parameters may be viewed as RVs and the variance of the target label may change between observations. 

\subsection{Extended Kalman Filter (EKF)}
\label{Section:EKF}
In this section we briefly outline the Extended Kalman filter \cite{anderson1979optimal, gelb1974applied}. The EKF is a standard technique for estimating the state of a nonlinear dynamic system or for learning the parameters of a nonlinear approximation function. In this paper we will focus on its latter role, meaning estimating $\boldsymbol{\theta}$. The EKF considers the following model:
\begin{equation}
\label{eq:Extended-Kalman}
\begin{cases}
\boldsymbol{\theta}_t = \boldsymbol{\theta}_{t-1} + {\bf v}_t\\
y({\bf u}_t) = h({\bf u}_t; \boldsymbol{\theta}_t) + {\bf n}_t
\end{cases},
\end{equation}
where $\boldsymbol{\theta}_t \in \mathbb{R}^{d \times 1}$ are the parameters evaluated at time $t$, $y ({\bf u}_t)$ is the $N$-dimensional observations vector at time $t$:
\begin{equation}
\label{eq:label_vector}
y ({\bf u}_t) = [ 
y(u_t^1),  y(u_t^2), \ldots, y(u_t^N)  ]^\top \in \mathbb{R}^{N \times 1},
\end{equation}
and $h({\bf u}_t; \boldsymbol{\theta}_t) \in \mathbb{R}^{N \times 1}$ is an $N$-dimensional vector, where $h(u; \boldsymbol{\theta})$ is a nonlinear observation function with input $u$ and parameters $\boldsymbol{\theta}$:
\begin{equation}
\label{eq:observation_func_vector}
h({\bf u}_t; \boldsymbol{\theta}_t) =[
h(u_t^1; \boldsymbol{\theta}_t), 
h(u_t^2; \boldsymbol{\theta}_t), 
\ldots,
h(u_t^N; \boldsymbol{\theta}_t)  
]^\top.
\end{equation} 
${\bf v}_t$ is the evolution noise, ${\bf n}_t$ is the observation noise, both modeled as additive and white noises with covariances ${\bf P}_{{\bf v}_t}$ and ${\bf P}_{{\bf n}_t}$,  respectively. As seen in the model presented in Equation (\ref{eq:Extended-Kalman}), EKF treats the parameters $\boldsymbol{\theta}_t$ as RVs, similarly to Bayesian approaches. According to this perspective, the parameters belong to an uncertainty set $\Theta$ governed by the mean and covariance of the parameters distribution. 

The estimation at time t, denoted as $\boldsymbol{\hat{\theta}}_{t|\cdot}$ is the conditional expectation of the parameters with respect to the observed data. The EKF formulation distinguishes between estimates that are based on observations up to time $t$, $\boldsymbol{\hat{\theta}}_{t|t}  \triangleq \mathbb{E} [ \boldsymbol{\theta}_t | y_{1:t}]$, and observations up to time $t-1$, $\boldsymbol{\hat{\theta}}_{t|t-1}  \triangleq \mathbb{E} [ \boldsymbol{\theta}_t | y_{1:t-1}]  = \boldsymbol{\hat{\theta}}_{t-1|t-1}$. With some abuse of notation, $y_{1:t'}$ are the observations gathered up to time $t'$: $y ({\bf u}_1), \ldots, y ({\bf u}_{t'})$. The {\it parameters errors} are defined by: $\boldsymbol{\tilde{\theta}}_{t|t}  \triangleq \boldsymbol{\theta}_t - \boldsymbol{\hat{\theta}}_{t|t}$ and $\boldsymbol{\tilde{\theta}}_{t|t-1} \triangleq \boldsymbol{\theta}_t - \boldsymbol{\hat{\theta}}_{t|t-1}$. The conditional {\it error covariances} are given by: ${\bf P}_{t|t}  \triangleq \mathbb{E} \big[  \boldsymbol{\tilde{\theta}}_{t|t}  \boldsymbol{\tilde{\theta}}_{t|t}^\top | y_{1:t}\big], \ 
{\bf P}_{t|t-1}  \triangleq \mathbb{E} \big[  \boldsymbol{\tilde{\theta}}_{t|t-1}  \boldsymbol{\tilde{\theta}}_{t|t-1}^\top | y_{1:t-1}\big] \\
 = {\bf P}_{t-1|t-1} + {\bf P}_{{\bf v}_t}$.

EKF considers several statistics of interest at each time step:
{\it The prediction of the observation function}, {\it the observation innovation}, {\it the covariance between the parameters error and the innovation}, {\it the covariance of the innovation}  and the {\it Kalman gain} are defined respectively in Equations (\ref{eq:The prediction of the observation}) - (\ref{eq:Kal_gain}):
\begin{align}
\label{eq:The prediction of the observation} & {\bf \hat{y}}_{t|t-1} \triangleq  \mathbb{E}[h({\bf u}_t; \boldsymbol{\theta}_t)|y_{1:t-1}],\\
\label{eq:The observation innovation} & {\bf \tilde{y}}_{t|t-1} \triangleq h({\bf u}_t; \boldsymbol{\theta}_t) - {\bf \hat{y}}_{t|t-1}, \\
\label{eq:The covariance between the weights error and the innovation} & {\bf P}_{\boldsymbol{\tilde{\theta}}_t,{\bf \tilde{y}}_t}  \triangleq \mathbb{E}[ \boldsymbol{\tilde{\theta}}_{t|t-1} {\bf \tilde{y}}_{t|t-1} |y_{1:t-1}],\\	
\label{eq:The covariance of the innovation} & {\bf P}_{{\bf \tilde{y}}_t}  \triangleq \mathbb{E}[ {\bf \tilde{y}}_{t|t-1} {\bf \tilde{y}}_{t|t-1}^\top |y_{1:t-1}] + {\bf P}_{{\bf n}_t},\\
\label{eq:Kal_gain} & {\bf K}_t \triangleq {\bf P}_{\boldsymbol{\tilde{\theta}}_t,{\bf \tilde{y}}_t} {\bf P}_{{\bf \tilde{y}}_t}^{-1}.
\end{align}
The above statistics serve for the EKF updates:
\begin{equation}
\label{eq:EKF}
\begin{cases}
\boldsymbol{\hat{\theta}}_{t|t}^{\text{EKF}} = \boldsymbol{\hat{\theta}}_{t|t-1} +  {\bf K}_t \big( y({\bf u}_t) - h({\bf u}_t; \boldsymbol{\hat{\theta}}_{t|t-1}) \big),\\
{\bf P}_{t|t} = {\bf P}_{t|t-1} - {\bf K}_t {\bf P}_{{\bf \tilde{y}}_t} {\bf K}_t^\top.
\end{cases}
\end{equation}
In the next section we present how to use the EKF formulation in order to approximate VFs which consider uncertainty both in the parameters and in the noisy observations.
\section{EKF for Value Function Approximation}
\label{Section:EKF_RL}
We now derive a novel regularized objective function and argue in its favor for optimizing value functions in RL. We use general notations in order to enable integration of our proposed VF optimization method with any policy optimization algorithm. The main idea is to decompose the Bellman TD error vector $\delta({\bf u}_t; \boldsymbol{\theta}_t)$ into two parts: \\
$\delta({\bf u}_t; \boldsymbol{\theta}_t)  = y({\bf u}_t) - h({\bf u}_t; \boldsymbol{\theta}_t) = [\delta(u_t^1; \boldsymbol{\theta}_t), .., \delta(u_t^N; \boldsymbol{\theta}_t)]^\top$. (i) The observation at time $t$, $y({\bf u}_t)$ is a vector that contains $N$ target labels $y (u_t^1), \ldots, y (u_t^N)$. (ii) The observation function may be one of the following:
\begin{equation*}
h(u; \boldsymbol{\theta}_t) = 
\begin{cases}
V(s; \boldsymbol{\theta}_t) & \text{the state value function}\\
Q(s, a; \boldsymbol{\theta}_t) & \text{the state-action Q-function}\\
A(s, a; \boldsymbol{\theta}_t) & \text{the advantage function.}
\end{cases}
\end{equation*}
The observation functions for $N$ inputs are concatenated into the $N$-dimensional vector $h({\bf u}_t; \boldsymbol{\theta}_t)$, as presented in Equation (\ref{eq:observation_func_vector}).  In Table \ref{vf_table} we provide several examples for the Bellman TD error decomposition according to the chosen policy optimization algorithm.

Our goal is to estimate the parameters $\boldsymbol{\theta}_t$. One way is to learn them by maximum likelihood estimation (MLE) using stochastic gradient descent methods: $\boldsymbol{\theta}^{\text{MLE}} = \arg\max_{\boldsymbol{\theta}} \log p(y_{1:t}|\boldsymbol{\theta})$. This forms the objective function in Equation (\ref{eq:vf objective}). Another way is learning them by a Bayesian approach which uses Bayes rule and adds prior knowledge over the parameters $p(\boldsymbol{\theta})$ to calculate the maximum a-posteriori (MAP) estimator: $\boldsymbol{\theta}^{\text{MAP}}  = \arg\max_{\boldsymbol{\theta}} \log p(\boldsymbol{\theta}|y_{1:t}) 
 = \arg\max_{\boldsymbol{\theta}} \log p(y_{1:t}|\boldsymbol{\theta}) + \log p(\boldsymbol{\theta})$. Given the observations gathered up to time $t$, we can re-write the MAP estimator:
\begin{align}
\label{eq:MAP_EKF}
\boldsymbol{\theta}_t^{\text{MAP}} = \arg\max_{\boldsymbol{\theta}_t} \log p(y_t|\boldsymbol{\theta}_t) + \log  p(\boldsymbol{\theta}_t|y_{1:t-1}).
\end{align}
Here, instead of using the parameters prior, we use an equivalent derivation for the parameters posterior conditioned on $y_{1:t}$, based on the likelihood of a {\it single} observation $y_t \triangleq y({\bf u}_t)$ and the posterior conditioned on $y_{1:t-1}$ \cite{van2004sigma}. This unique derivation is a key step for making the incremental Kalman updates and for defining the objective function in Equation (\ref{eq:EKF-loss}). In order do define the likelihood  $p(y_{t}|\boldsymbol{\theta}_t)$ and the posterior $p(\boldsymbol{\theta}_t|y_{1:t-1})$, we adopt the EKF model (\ref{eq:Extended-Kalman}), and make the following assumptions: 

\begin{assumption}
	\label{As:ConditionalIndependance2}
	The likelihood $p(y({\bf u}_t)|\boldsymbol{\theta}_t)$ is assumed to be Gaussian: 
	$y({\bf u}_t)|\boldsymbol{\theta}_t \sim \mathcal{N}( h({\bf u}_t; \boldsymbol{\theta}_t), {\bf P}_{{\bf n}_t})$ 
\end{assumption}
\begin{assumption}
	\label{As:GaussianPosterior2}
	The posterior distribution $p(\boldsymbol{\theta}_t|y_{1:t-1})$ is assumed to be Gaussian: $\boldsymbol{\theta}_t|y_{1:t-1} \sim \mathcal{N}(\boldsymbol{\hat{\theta}}_{t|t-1},{\bf P}_{t|t-1})$.
\end{assumption}
These assumptions are common when using the EKF. In the context of RL, these assumptions add the flexibility we want: the value is treated as a RV and information is gathered on the uncertainty of its estimate. In addition, the noisy observations (the target labels), can have different variances and can even be correlated. Based on these Gaussian assumptions, we can derive the following Theorem:
\begin{theorem}
	\label{theorem1}
	Under Assumptions \ref{As:ConditionalIndependance2} and \ref{As:GaussianPosterior2}, $\boldsymbol{\hat{\theta}}^{\text{EKF}}_{t|t}$  (\ref{eq:EKF}) minimizes at each $t$ the following regularized objective function:
	\begin{align}
	\label{eq:EKF-loss}
	\nonumber L^{\text{EKF}}_t(\boldsymbol{\theta}_t) & =  \frac{1}{2}  \delta({\bf u}_t; \boldsymbol{\theta}_t)^\top {\bf P}_{{\bf n}_t}^{-1}  \delta({\bf u}_t; \boldsymbol{\theta}_t) \\
	& +  \frac{1}{2}(\boldsymbol{\theta}_t - \boldsymbol{\hat{\theta}}_{t|t-1})^\top {\bf P}_{t|t-1}^{-1} (\boldsymbol{\theta}_t - \boldsymbol{\hat{\theta}}_{t|t-1}),
	\end{align}
	where $\boldsymbol{\hat{\theta}}^{\text{EKF}}_{t|t} \in \arg\min_{\boldsymbol{\theta}_t} L^{\text{EKF}}_t(\boldsymbol{\theta}_t)$. 
\end{theorem}
The proof for Theorem \ref{theorem1} appears in the supplementary material. It is based on solving the maximization problem in Equation (\ref{eq:MAP_EKF}) using the EKF model (\ref{eq:Extended-Kalman}) and the Gaussian Assumptions \ref{As:ConditionalIndependance2} and \ref{As:GaussianPosterior2}.

We now explicitly write the expressions for the statistics of interest in Equations (\ref{eq:The prediction of the observation}) - (\ref{eq:Kal_gain})  (see the supplementary material for more detailed derivations). The derivations are based on the first order Taylor series linearization for the observation function $h(\boldsymbol{\theta}_t)$:
$h({\bf u}_t; \boldsymbol{\theta}_t) 
= h({\bf u}_t; \boldsymbol{\hat{\theta}}) +   \nabla_{\boldsymbol{\theta}_t} h({\bf u}_t; \boldsymbol{\hat{\theta}})^\top \big( \boldsymbol{\theta}_{t} - \boldsymbol{\hat{\theta}} \big)$, where 
\begin{align}
\label{eq:nabla_h}
& \nabla_{\boldsymbol{\theta}_t} h({\bf u}_t; \boldsymbol{\hat{\theta}}) \\
\nonumber & = \begin{bmatrix} 
\nabla_{\boldsymbol{\theta}_t} h(u_t^1; \boldsymbol{\hat{\theta}}) , \nabla_{\boldsymbol{\theta}_t} h(u_t^2; \boldsymbol{\hat{\theta}}), \ldots , \nabla_{\boldsymbol{\theta}_t} h(u_t^N; \boldsymbol{\hat{\theta}})  \end{bmatrix} \in \mathbb{R}^{d \times N}
\end{align}
and $\boldsymbol{\hat{\theta}}$ is typically chosen to be the previous estimation of the parameters at time $t-1$,  $\boldsymbol{\hat{\theta}}=\boldsymbol{\hat{\theta}}_{t|t-1}$. {\it The prediction of the observation function} is ${\bf \hat{y}}_{t|t-1} = h({\bf u}_t; \boldsymbol{\hat{\theta}})$, {\it the covariance between the parameters error and the innovation} is ${\bf P}_{\boldsymbol{\tilde{\theta}}_t,{\bf \tilde{y}}_t}  = {\bf P}_{t|t-1} \nabla_{\boldsymbol{\theta}_t} h({\bf u}_t,  \boldsymbol{\hat{\theta}})$ and {\it the covariance of the innovation} is:
\begin{align}
{\bf P}_{{\bf \tilde{y}_t}}   = \nabla_{\boldsymbol{\theta}_t} h({\bf u}_t; \boldsymbol{\hat{\theta}})^\top {\bf P}_{t|t-1} \nabla_{\boldsymbol{\theta}_t} h({\bf u}_t; \boldsymbol{\hat{\theta}})  + {\bf P}_{{\bf n}_t}.
\label{eq:statistics}
\end{align}
The Kalman gain then becomes:
\begin{align}
\label{eq:Kalman_gain}
\nonumber {\bf K}_t & = {\bf P}_{t|t-1} \nabla_{\boldsymbol{\theta}_t} h({\bf u}_t,  \boldsymbol{\hat{\theta}}) \\
& \big( \nabla_{\boldsymbol{\theta}_t} h({\bf u}_t; \boldsymbol{\hat{\theta}})^\top {\bf P}_{t|t-1} \nabla_{\boldsymbol{\theta}_t} h({\bf u}_t; \boldsymbol{\hat{\theta}})  + {\bf P}_{{\bf n}_t} \big)^{-1}.
\end{align}
This Kalman gain is used in the parameters update and the error covariance update in Equation (\ref{eq:EKF}).

\subsection{Comparing between $L_t^{\text{EKF}}$ and $L_t^{\text{MLE}}$ for Optimizing Value Functions}
\label{Section:comparison}
We argue in favor of using the regularized objective function $L^{\text{EKF}}_t(\boldsymbol{\theta}_t)$ (\ref{eq:EKF-loss}) for optimizing VFs instead of the commonly used objective function $L^{\text{MLE}}_t(\boldsymbol{\theta}_t)$ (\ref{eq:vf objective}). Corollary \ref{corollary1} will assist us to discuss and compare between the two objective functions:
\begin{corollary}
	\label{corollary1}
	Under Assumptions \ref{As:ConditionalIndependance2} and \ref{As:GaussianPosterior2}, consider a diagonal covariance ${\bf P}_{{\bf n}_t}$ with diagonal elements $\sigma_i = N$  and assume ${\bf P}_{0|0} = {\bf P}_{{\bf v}_t} = {\bf 0}$, then:
	$L^{\text{EKF}}_t(\boldsymbol{\theta}_t) = L^{\text{MLE}}_t(\boldsymbol{\theta}_t)$.
\end{corollary}
The proof is given in the supplementary material. According to Corollary \ref{corollary1}, the two objective functions are the same if we consider the parameters as deterministic and if we assume that the noisy target labels have a fixed variance. 

So what are the differences between the two objective functions? First, $L^{\text{EKF}}_t$ is a regularized version of $L^{\text{MLE}}_t$: the regularization is causing the parameters $\boldsymbol{\theta}_t$ to {\it track} the recent parameters estimate, $\boldsymbol{\hat{\theta}}_{t|t-1}$, stabilizing the estimate process. The error between the successive estimates is weighted with the inverse of the uncertainty information ${\bf P}_{t|t-1}$. $L^{\text{MLE}}_t$ does not include a regularization term, meaning it does not account for parametrization uncertainties. Note that when adding a standard $L_2$ regularization to $L^{\text{MLE}}_t$, often common in DNNs, it reflects staying close to the ${\bf 0}$ vector which is not always desired. 

Second, $L^{\text{EKF}}_t$ weights the squared Bellman TD error vector $\delta({\bf u}_t; \boldsymbol{\theta}_t)$ with ${\bf P}_{{\bf n}_t}^{-1}$ which can be interpreted as an additional regularization technique. ${\bf P}_{{\bf n}_t}$ can be viewed as the amount of confidence we have in the observations, as defined in the EKF model (\ref{eq:Extended-Kalman}): if the observations are noisy, we should consider larger values for the diagonal elements in the covariance ${\bf P}_{{\bf n}_t}$. In addition, $L^{\text{EKF}}_t$ allows us to model correlations between observations errors, unlike the iid assumption in $L^{\text{MLE}}_t$. In Section \ref{Section:experiments} we discuss possible options for ${\bf P}_{{\bf n}_t}$.  
 
Looking at the parameters update in Equation (\ref{eq:EKF}) and the definition of the Kalman gain ${\bf K}_t$ in Equation (\ref{eq:Kalman_gain}), we can see that the Kalman gain propagates the new information from the noisy target labels, back down into the parameters uncertainty set $\Theta$, before combining it with the estimated parameter value. Actually, ${\bf K}_t$ can be interpreted as an adaptive learning rate for each individual parameter that implicitly incorporates the uncertainty of each parameter. This approach resembles familiar stochastic gradient optimization methods such as Adagrad \cite{duchi2011adaptive}, AdaDelta \cite{zeiler2012adadelta}, RMSprop \cite{tieleman2012lecture}
and Adam \cite{kingma2014adam}, for different choices of ${\bf P}_{t|t-1}$ and ${\bf P}_{{\bf n}_t}$. We refer the reader to \citet{ruder2016overview}.

When looking at $L^{\text{EKF}}_t(\boldsymbol{\theta}_t)$ the reader may ask  what do $\boldsymbol{\hat{\theta}}_{t|t-1}$ and ${\bf P}_{t|t-1}$ stand for? When we are estimating the VF parameters for a fixed policy $\pi_{\text{old}}$, our objective function imposes a {\it trust region} method in each iteration of a batch optimization procedure. The trust region helps us to avoid over-fitting to the most recent batch of data. In this case $\boldsymbol{\hat{\theta}}_{t|t-1}^{\pi_{\text{old}}}$ is the last evaluation of the VF parameters for the same fixed policy $\pi_{\text{old}}$, and ${\bf P}_{t|t-1}^{\pi_{\text{old}}}$ is the conditional error covariance between the new parameters estimation and the previous one, again, for the same fixed policy\footnote{We added the upper-script $\pi_{\text{old}}$ to emphasis that the VF parameters correspond to evaluating the same policy.}. When we change policies and start to evaluate the VF parameters of $\pi_{\text{new}}$ we set $\boldsymbol{\theta}_{0|0}^{\pi_{\text{new}}} = \boldsymbol{\hat{\theta}}_{t|t}^{\pi_{\text{old}}}$ and ${\bf P}_{0|0}^{\pi_{\text{new}}} = {\bf P}_{t|t}^{\pi_{\text{old}}}$, meaning we start a new estimation procedure at $t=0$ for the new policy.

\subsection{Connection between EKF, Natural Gradient and the Gauss-Newton Method}
The EKF may be viewed as an on-line natural gradient algorithm \cite{amari1998natural} that uses the Fisher information matrix ${\bf J}_t$ \cite{ollivier2018online}. In this setting, the connection between the error covariance matrix and the Fisher matrix is given by: ${\bf P}_{t|t}^{-1} = (t+1) {\bf J}_t$. This insight suggests that the regularization term in $L_t^{\text{EKF}}$ is actually a second order approximation of the KL-divergence between the previous parameter estimate and the current one. Combining these insights together, we conclude that our proposed method can be viewed as a natural gradient algorithm for VF approximation. Similarly, the EKF may be viewed as an incremental version of the Gauss-Newton method, which is a common iterative method for solving least squares problems \cite{bertsekas1996incremental}. When updating the parameters, the Gauss-Newton uses the matrix $H = \mathbb{E}[J^T J]$ where $J$ is the Jacobian of $h(\boldsymbol{\theta}_t)$. When the observations are assumed to be Gaussian (as we assume in Assumption \ref{As:ConditionalIndependance2}), $H$ is equivalent to the Fisher information matrix. 

The following Theorem formalizes the connection between EKF and two separate KL divergences:
\begin{theorem}
	\label{theorem:KL}
	Assume the inputs $u$ are drawn independently from a training distribution $\hat{Q}_{u}$ with density function $q(u)$, and assume the corresponding observations $y$ are drawn from a conditional training distribution $\hat{Q}_{y|u}$ with density function $q(y|u)$. 
	Let $Q_{u,y}$ be the joint distribution whose density is $q(u,y) = q(y|u)q(u)$, and let $P_{u,y}(\boldsymbol{\theta})$ be the learned distribution, whose density is $p(u,y|\boldsymbol{\theta}) = p(y|u, \boldsymbol{\theta})q(u)$. Under Assumptions \ref{As:ConditionalIndependance2} and \ref{As:GaussianPosterior2}, consider a diagonal covariance ${\bf P}_{{\bf n}_t}$ with diagonal elements $\sigma_i = N$ then:
	\begin{align*}
	 L^{\text{EKF}}_t(\boldsymbol{\theta}_t) & = C +  N \mathbb{E}_{\hat{Q}_u} [D_{\text{KL}} \big(\hat{Q}_{y|u} || P_{y|u}(\boldsymbol{\theta}) \big)]\\
	& +  t \cdot D_{\text{KL}} \big(P_{u,y}(\boldsymbol{\theta} + \Delta \boldsymbol{\theta})|| P_{u,y}(\boldsymbol{\theta}) \big)  + \mathcal{O}(\| \Delta \boldsymbol{\theta}\|^3)
	\end{align*}
	where $\small C = \log \big(\frac{1}{(2\pi)^{N/2} |{\bf P}_{{\bf n}_t}|^{1/2} } \big)$. 
\end{theorem}
Theorem \ref{theorem:KL} illustrates how EKF is aimed at minimizing two separate KL-divergences. The first is the KL divergence between two conditional distributions,  $\hat{Q}_{y|u}$ and $P_{y|u}(\boldsymbol{\theta})$. This term is equivalent to the loss in $L_t^{\text{MLE}}$ (\ref{eq:vf objective}). The second is the KL divergence between two different parameterizations of the joint learned distribution $P_{u,y}$. This is the term which imposes trust region on the VF parameters in $L_t^{\text{EKF}}$ (\ref{eq:EKF-loss}). The proof for Theorem \ref{theorem:KL} appears in the supplementary material. 
 
\subsection{Practical algorithm: KOVA optimizer} 
\label{section:KOVA}

\begin{figure}[t]
	\centering{
		\includegraphics[width=0.95\linewidth,height=0.5\textheight,keepaspectratio]{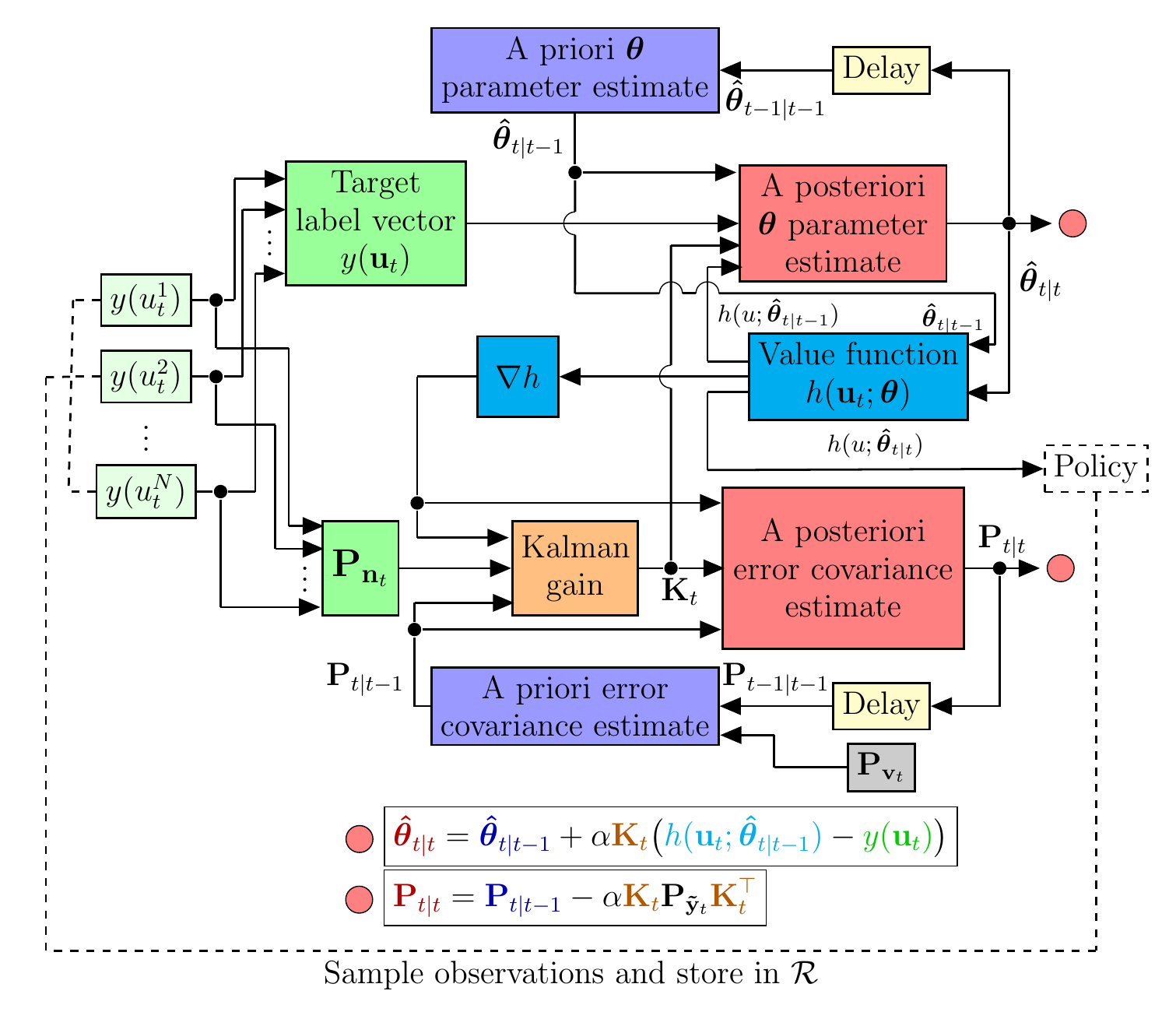}}
	\caption{KOVA optimizer block diagram. KOVA receives as input the initial general prior ${\bf P}_{0|0}$ and the covariances ${\bf P}_{{\bf v}_t}$ and ${\bf P}_{{\bf n}_t}$. It  initializes $\boldsymbol{\hat{\theta}}_{0|0}$ with small random values or with the VF parameters of the previous policy (see the discussion in Section \ref{Section:comparison}). For every $t$, it samples $N$ target labels from $\mathcal{R}$ (see Table \ref{vf_table} for target label examples), constructs $y({\bf u}_t)$ (\ref{eq:label_vector}) and $ h({\bf u}_t$, $\boldsymbol{\hat{\theta}}_{t|t-1})$ (\ref{eq:observation_func_vector}) and computes $\nabla_{\boldsymbol{\theta}_t} h({\bf u}_t; \boldsymbol{\hat{\theta}}_{t|t-1})$ (\ref{eq:nabla_h}) and ${\bf K}_t$  (\ref{eq:statistics})-(\ref{eq:Kalman_gain}). Then it updates and outputs the MAP parameters estimator $\boldsymbol{\hat{\theta}}_{t|t}$ and the error covariance matrix ${\bf P}_{t|t}$ according to Equation (\ref{eq:EKF}).}
	\label{fig:kalman_block_diagram}
\end{figure}

\begin{algorithm}[t]
	{\small 
	\caption{KOVA Optimizer}
	\label{alg:KOVA}
	\begin{algorithmic}[1]
		\REQUIRE ${\bf P}_{0|0}$, ${\bf P}_{{\bf v}_t}$,  ${\bf P}_{{\bf n}_t}$, $\alpha$,  $\mathcal{R}$. {\bf Initialize: } $\boldsymbol{\hat{\theta}}_{0|0}$, $t=0$.
		\FOR {$t=1, \ldots, T$}
		\STATE Set predictions:\\ $\begin{cases}
		\boldsymbol{\hat{\theta}}_{t|t-1} = \boldsymbol{\hat{\theta}}_{t-1|t-1}\\
		{\bf P}_{t|t-1} = {\bf P}_{t-1|t-1} + {\bf P}_{{\bf v}_t}
		\end{cases}$.
		\STATE Sample N tuples $\{y(u^i), h(u^i; \boldsymbol{\hat{\theta}}_{t|t-1})\}_{i=1}^N$ from $\mathcal{R}$.
		\STATE Construct $N$-dim vectors $y({\bf u}_t)$ (\ref{eq:label_vector}) and $ h({\bf u}_t,\boldsymbol{\hat{\theta}}_{t|t-1})$ (\ref{eq:observation_func_vector}).
		\STATE Compute $(d \times N)$-dim matrix $\nabla_{\boldsymbol{\theta}} h({\bf u}_t; \boldsymbol{\hat{\theta}}_{t|t-1})$ (\ref{eq:nabla_h}).
		\STATE ${\bf P}_{\boldsymbol{\tilde{\theta}},{\bf \tilde{y}}_t}  = {\bf P}_{t|t-1} \nabla_{\boldsymbol{\theta}} h({\bf u}_t,  \boldsymbol{\hat{\theta}}_{t|t-1})$.
		\STATE ${\bf P}_{{\bf \tilde{y}_t}} = \nabla_{\boldsymbol{\theta}} h({\bf u}_t; \boldsymbol{\hat{\theta}}_{t|t-1})^\top {\bf P}_{t|t-1} \nabla_{\boldsymbol{\theta}} h({\bf u}_t; \boldsymbol{\hat{\theta}}_{t|t-1})  + {\bf P}_{{\bf n}_t}$.
		\STATE ${\bf K}_t = {\bf P}_{\boldsymbol{\tilde{\theta}}_t,{\bf \tilde{y}}_t} {\bf P}_{{\bf \tilde{y}_t}}^{-1}$
		\STATE Set updates:\\ $\begin{cases}
		\boldsymbol{\hat{\theta}}_{t|t} = \boldsymbol{\hat{\theta}}_{t|t-1} + \alpha {\bf K}_t \big( y({\bf u}_t) - h({\bf u}_t; \boldsymbol{\hat{\theta}}_{t|t-1}) \big)\\
		{\bf P}_{t|t} = {\bf P}_{t|t-1} - \alpha {\bf K}_t {\bf P}_{{\bf \tilde{y}}_t} {\bf K}_t^\top
		\end{cases}$
		\ENDFOR
		\ENSURE $\boldsymbol{\hat{\theta}_{t|t}}$ and ${\bf P}_{t|t}$
	\end{algorithmic}
}
\end{algorithm}
We now derive a practical algorithm for approximating VFs, by minimizing the objective function $L^{\text{EKF}}_t$ (\ref{eq:EKF-loss}). In practice we use the update Equations (\ref{eq:EKF}) and the Kalman gain Equation in (\ref{eq:statistics})-(\ref{eq:Kalman_gain}) in order to avoid inversing ${\bf P}_{t|t-1}$. In addition, we add a fixed learning rate $\alpha$ to smooth the update. The KOVA optimizer is presented in Algorithm \ref{alg:KOVA} and illustrated in Figure \ref{fig:kalman_block_diagram}. Notice that $\mathcal{R}$ is a samples generator whose structure depends on the policy algorithm for which KOVA is used as a VF optimizer. $\mathcal{R}$ can contain trajectories from a fixed policy or it can be an experience replay which contains transitions from several different policies. 

{\bf Algorithm complexity:} For a $d$-dimensional parameter vector $\boldsymbol{\theta} \in \mathbb{R}^d$, our algorithm requires $\mathcal{O}(d^2)$ extra space to store the covariance matrix and $\mathcal{O}(d^2)$ computations for matrix multiplications. Note that our update method does not require inverting the $\small (d \times d)$-dimensional matrix ${\bf P}_{t|t-1}$ in the update process, but only requires inverting the  $\small (N\times N)$-dimensional matrix $\small \big( \nabla h(\boldsymbol{\hat{\theta}})^\top {\bf P}_{t|t-1} \nabla h( \boldsymbol{\hat{\theta}}) + {\bf P}_{{\bf n}_t} \big)^{-1}$. Usually, $N \ll d$. 
The extra time and memory requirements can be tolerated for small-medium networks with size $d$. However, it can be considered as a drawback of the algorithm for large network sizes. Fortunately, there are several options for overcoming these drawbacks: {\bf (a)} The use of GPU for matrix multiplications can accelerate the computation time. {\bf (b)} We can assume correlations only between blocks of parameters, for example, between parameters in the same DNN layer, and apply layer factorization. This can reduce significantly the computation and memory requirements \cite{puskorius1991decoupled,zhang2017learning,wu2017scalable}. {\bf (c)} We can apply the Kalman optimization method only on the last layer in large DNNs. This approach was used by \citet{levine2017shallow} where they optimized the last layer using linear least squares optimization methods. 
We emphasis that yet, our approach scales with large state and action spaces, and is suitable for continuous control problems which are considered hard domains. 

\section{Related Work}

\begin{figure*}[htp]
	\centering
	\subfigure[]{\label{fig:mujoco_reward_ppo}\includegraphics[width=0.49\linewidth,height=0.5\textheight,keepaspectratio]{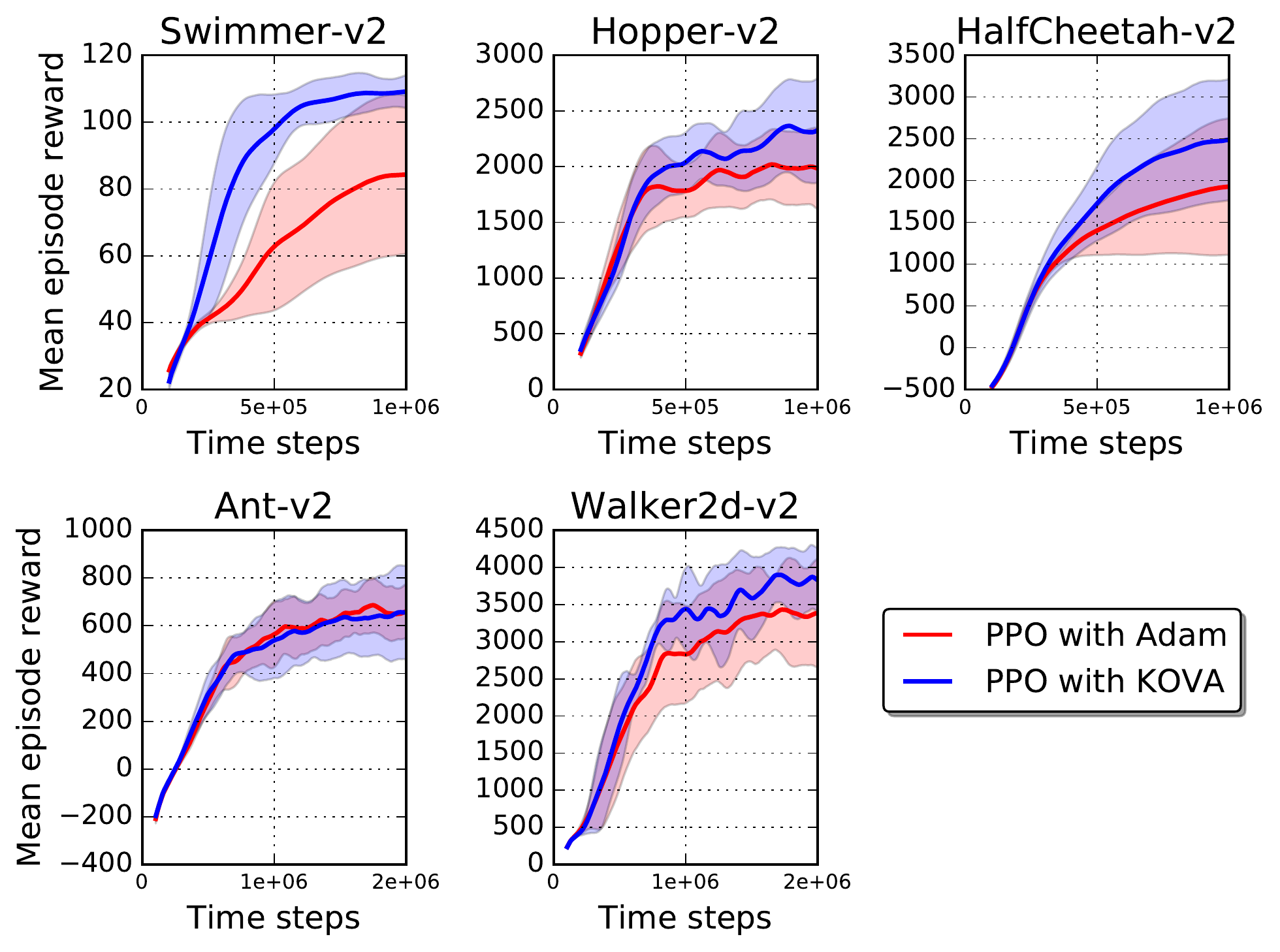}}
	\subfigure[]{\label{fig:mujoco_reward_trpo}\includegraphics[width=0.49\linewidth,height=0.5\textheight,keepaspectratio]{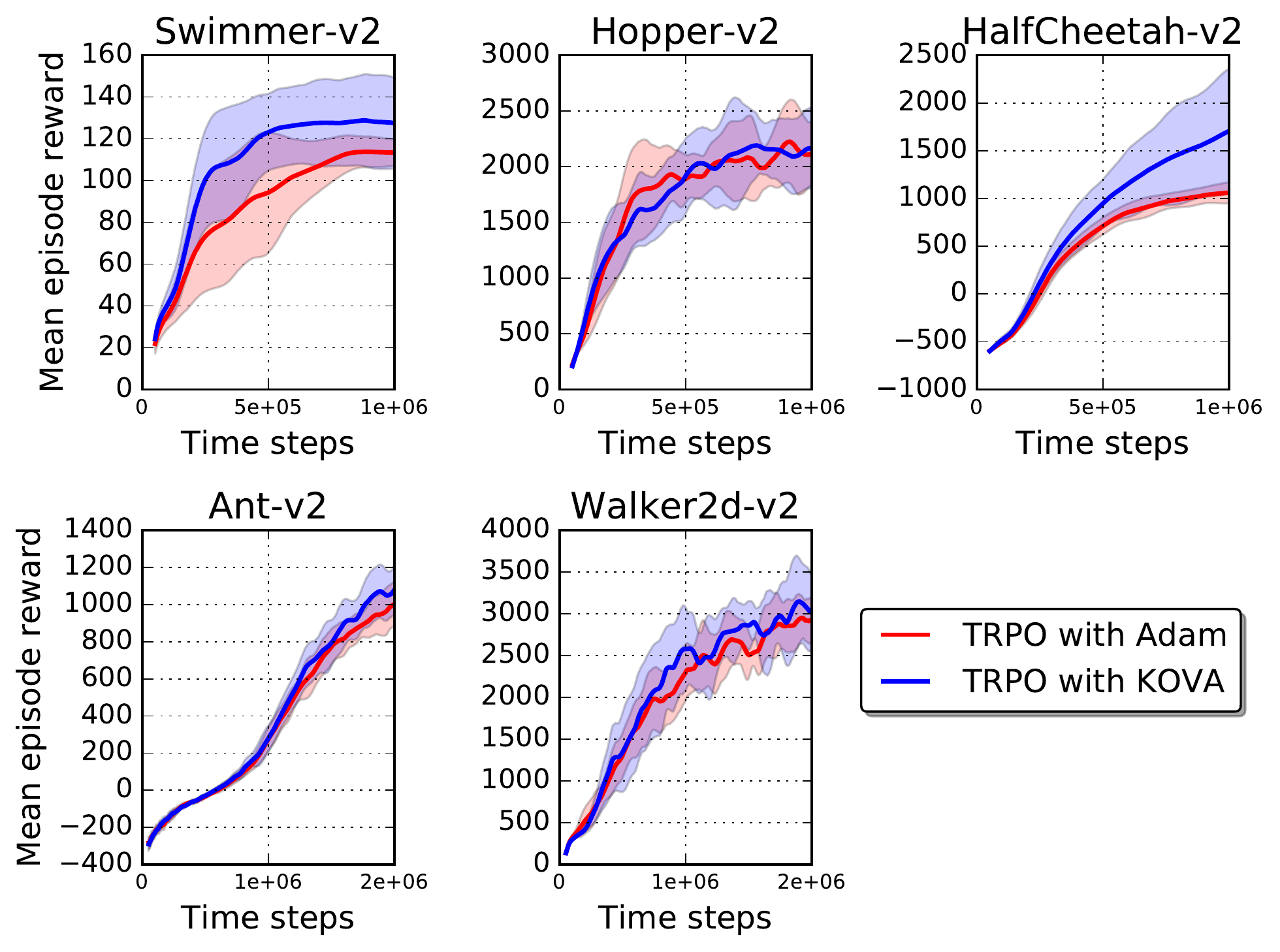}}
	\caption{Mean episode reward during training for Mujoco environments. {\bf (a)} PPO or {\bf (b)} TRPO are used as policy optimization algorithms. We compare between Adam and KOVA optimizers for policy evaluation. For Swimmer-v2, Hopper-v2 and HalfCheetah-v2 we trained over one million time steps and for  Ant-v2 and Walker2d-v2 we trained over two million time steps. We present the average (solid lines) and standard deviation (shaded area) of the episodes rewards over 8 runnings, generated from random seeds.}
	\label{fig:mujoco_reward}
\end{figure*}

{\bf Bayesian Neural Networks (BNNs):} There are several works on Bayesian methods for placing uncertainty on the approximator parameters \cite{blundell2015weight,gal2016dropout}. \citet{depeweg2016learning,depeweg2017decomposition} have used BNNs for learning MDP dynamics in RL tasks. In these works a fully factorized Gaussian distribution on parameters is assumed  while we consider possible correlations between parameters. In addition, BNNs require sampling the parameters, and running several feed-forward runs for each of the parameters samples. Our incremental method avoids multiple samples of the parameters, since the uncertainty is propagated with every optimization update. 

{\bf Kalman filters:} Outside of the RL framework, the use of Kalman filter as an optimization method is discussed in \cite{haykin2001kalman,vuckovic2018kalman,gomez2018decoupled}. \citet{wilson2009neural} solve the dynamics of each parameter with Kalman filtering. \citet{wang2018batch} use Kalman filter for normalizing batches. In our work we use Kalman filtering for VF optimization in the context of RL. EKF is connected with the incremental Gauss-Newton method \cite{bertsekas1996incremental}, and with the on-line natural gradient \cite{ollivier2018online}. These methods require inversing the $(d \times d)$-dimensional Fisher information matrix (for $d$-dimensional parameter), thus require high computational resources. Our method avoids this inversion in the update step which is more computationally efficient. 

{\bf Trust region for policies:} The natural gradient method, when applied to RL tasks, is mostly used in policy gradient algorithms to estimate the parameters of the policy \cite{kakade2002natural,peters2008natural,schulman2015trust}. Trust region methods in RL have been developed for parameterized policies \cite{schulman2015trust,schulman2017proximal}. Despite that, trust region methods for parametrized VFs are rarely presented in the RL literature. Recently, \citet{wu2017scalable} suggested to apply the natural gradient method also on the critic in the actor-critic framework, using Kronecker-factored approximations. \citet{schulman2015high} suggested to apply Gauss-Newton method to estimate the VF. However, they did not analyze and formalize the underlying model and assumptions that lead to the regularization in the objective function, while this is the focus in our work.

\begin{figure*}[htp]
	\centering
	\subfigure[]{\label{fig:entropy_ppo_max_ratio}\includegraphics[width=0.49\linewidth,height=0.5\textheight,keepaspectratio]{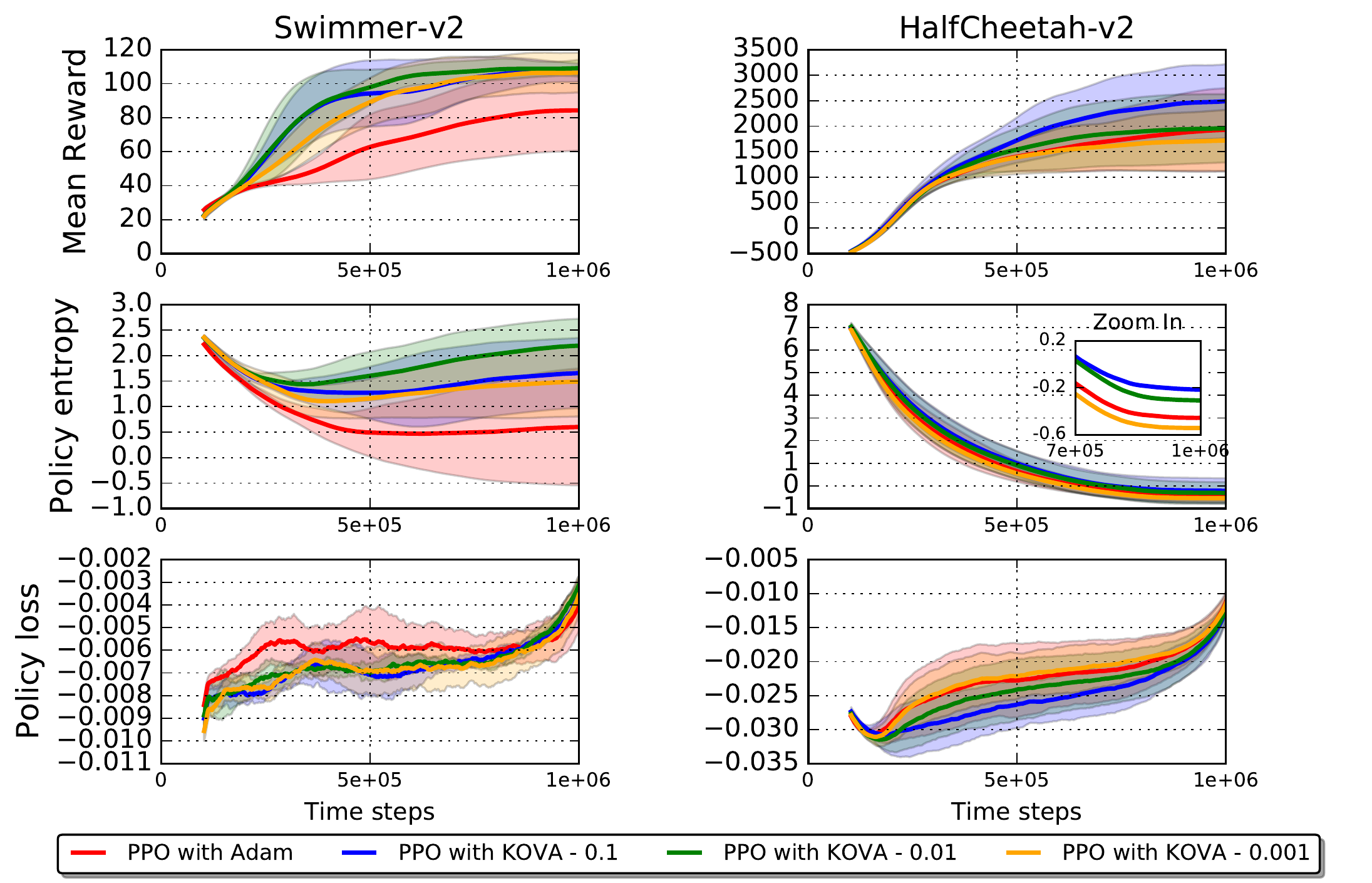}}
	\subfigure[]{\label{fig:entropy_ppo_batch_size}\includegraphics[width=0.49\linewidth,height=0.5\textheight,keepaspectratio]{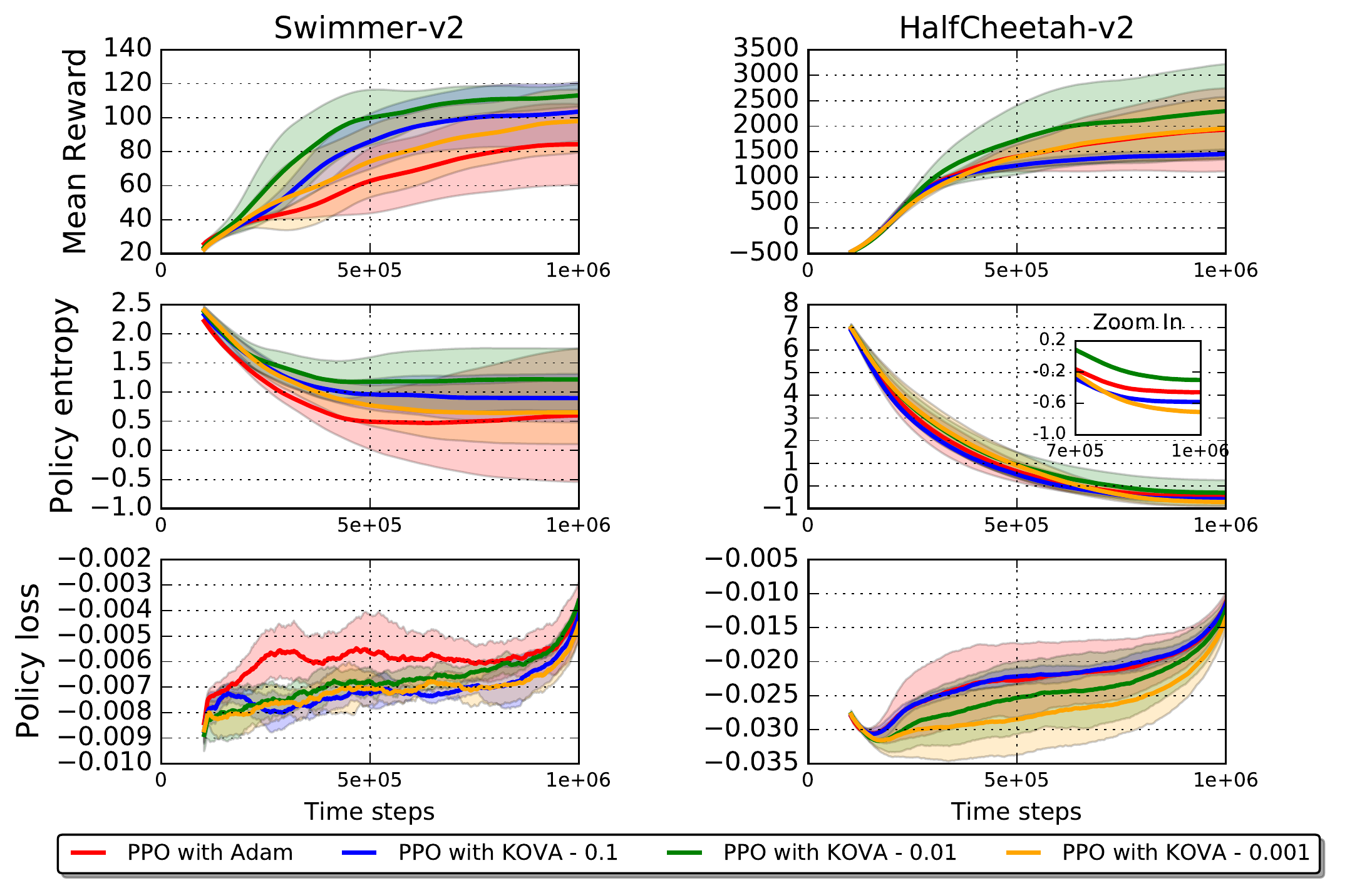}}
	\caption{Mean episode reward, policy entropy and the policy loss for a PPO agent in the Mujoco environments Swimmer-v2 and HalfCheetah-v2. We compare between optimizing the VF  with Adam vs. our KOVA optimizer. For KOVA, we present three different values for $\eta = 0.1, 0.01, 0.001$ and two different values for the diagonal elements in ${\bf P}_{{\bf n}_t}$: {\bf (a)} max-ratio and {\bf (b)} batch-size. We present the average (solid lines) and standard deviation (shaded area) of the episodes rewards over 8 runnings, generated from random seeds.} 
	\label{fig:entropy_ppo}
\end{figure*}
{\bf Distributional perspective on values and observations:} Distributional RL \cite{bellemare2017distributional} treats the full (general) distribution of total return, and considers VF parameters as deterministic. In our work we assume Gaussian distribution over the total return and in addition Gaussian distribution over the VF parameters.

Our work may be seen as a modern extension of GPTD \cite{engel2003bayes,engel2005reinforcement} for DRL domains with continuous state and action spaces. GPTD uses Gaussian Processes (GPs) for both VF and total return, for solving the RL problem of value estimation. We introduce here several improvements and generalizations over their work: (1) Our formulation is adapted to learning nonlinear VF approximations, as common in DRL; (2) We include a fading memory option for previous observations by using a decay factor in the error covariance prediction (${\bf P}_{{\bf v}_t}$); (3) We allow for a general observation noise covariance (not necessarily diagonal) and for a general noisy observations (not only 1-step TD errors); (4) Our observation vector $y({\bf u})$ has a fixed size $N$ (the batch size) as opposed to the growing size vectors in GPTD which grow for any new observation and make it difficult to train in DRL domains.

The use of Kalman filters to solve RL tasks was proposed by \citet{geist2010kalman}. Their formulation, called Kalman Temporal Difference (KTD), serves as the base for our formulation for the optimizer we propose. We introduce here several differences between their work and ours: (1) We re-formulate the observation equation (\ref{eq:EKF}) to increase training stability by using a target network for the VF that appears in the target label (see Table \ref{vf_table}). With this formulation, the observation function is simply the VF of the current input; (2) We use the Extended Kalman filter as opposed to their use of the Unscented Kalman filter to approximate nonlinear functions \cite{julier1997new, wan2000unscented}. In our formulation, the observation function is differential, allowing us to use first order Taylor expansion linearization. The UKF has shown superior performance in some applications \cite{st2004comparison, van2004sigma}, however, its computational cost is much greater than the computational cost of the EKF, due to its requirement of sampling the parameters in each training step for $2d$ times. Moreover, it requires to evaluate the observation function at these samples at every training step. Unfortunately, this is not tractable in DNNs where the parameters might be high-dimensional.

\section{Experiments}
\label{Section:experiments}
In this section we present experiments that illustrate the performance attained by our KOVA optimizer\footnote{Code can be found in: https://github.com/KOVA-trustregion/KOVA. Technical details on policy and VF networks and on the hyper-parameters we used are described in the supplementary material.}. 

{\bf KOVA optimizer for policy evaluation:} We tested the performance of KOVA in domains with high state and action spaces: the robotic tasks benchmarks implemented in OpenAI Gym \cite{brockman2016openai}, which use the MuJoCo physics engine \cite{todorov2012mujoco}. For the policy training we used  PPO \cite{schulman2017proximal} and TRPO \cite{schulman2015trust} and used their baselines implementations \cite{baselines}. For VF training we replaced the originally used Adam optimizer \cite{kingma2014adam} with our KOVA optimizer (Algoritm \ref{alg:KOVA}) and compared their affect on the mean episode reward in each environment. The results are presented in Figure \ref{fig:mujoco_reward}. When training with PPO, we can see that KOVA improved the agent's performance in four out of five environments. In Ant-v2 it kept approximately the same performance. When training with TRPO, we can see that KOVA improved the agent's performance mostly in Swimmer-v2 and HalfCheetah-v2. These improvements, both in PPO and in TRPO,  demonstrate the importance of incorporating uncertainty estimation in value function approximation for improving the agent's performance. 

{\bf Investigating the evolution and observation noises: } The most interesting hyper-parameters in KOVA are related to the covariances ${\bf P}_{{\bf v}_t}$ and ${\bf P}_{{\bf n}_t}$. As seen in Corollary \ref{corollary1}, for deterministic interpretation of the parameters we simply set ${\bf P}_{{\bf v}_t} = {\bf 0}$. However, the more interesting setting would be ${\bf P}_{{\bf v}_t} = \frac{\eta}{1-\eta}{\bf P}_{t-1|t-1}$ with $\eta$ being a small number that controls the amount of fading memory \cite{ollivier2018online}. ${\bf P}_{{\bf n}_t}$ can be used for incorporating prior domain knowledge. For example, a diagonal matrix implies independent observations , while if observations are known to be correlated, additional non-diagonal elements can be added. We investigated the effect of different values of $\eta$ and ${\bf P}_{{\bf n}_t}$ in the Swimmer and HalfCheetah environments, where KOVA gained the most success. The results are depicted in Figure \ref{fig:entropy_ppo}. We tested two different ${\bf P}_{{\bf n}_t}$ settings: the {\it batch-size} setting where $\sigma_i=\sigma=N$ and the {\it max-ratio} setting where $\sigma_i=N \max (1, \frac{1}{\frac{\pi_{\text{old}} (a_i|s_i)}{\pi_{\text{new}}(a_i|s_i)} + \epsilon})$. Interestingly, although using KOVA results in lower policy loss (which we try to maximize), it actually increases the policy entropy and encourages exploration, which we believe helps in gaining higher rewards during training. We can clearly see how the mean rewards increases as the policy entropy increases, for different values of $\eta$. This insight was observed in both tested Mujoco environments and in both settings of ${\bf P}_{{\bf n}_t}$.

\section{Conclusion}
In this work we presented a novel regularized objective function for optimizing VFs in policy evaluation, which originates from a Bayesian perspective over both noisy observations and value parameters. Our empirical results illustrate how the KOVA optimizer can improve the performance of various RL agents in domains with large state and action spaces. For future work, it would be interesting to further investigate the connection between trust region over value parameters and trust region over policy parameters and how to use this connection to improve exploration. 

\bibliographystyle{icml2019}
\bibliography{KalmanRLArxiv}

\newpage
\appendix
\section*{Supplementary Material}
\numberwithin{equation}{section}
\counterwithin{table}{section}
\section{Theoretical Results}
\subsection{Extended Kalman Filter (EKF)}
\label{suppSection:EKF}
In this section we briefly outline the Extended Kalman filter \cite{anderson1979optimal, gelb1974applied}. The EKF considers the following model:
\begin{equation}
\label{suppeq:Extended-Kalman}
\begin{cases}
\boldsymbol{\theta}_t = \boldsymbol{\theta}_{t-1} + {\bf v}_t\\
y ({\bf u}_t) = h({\bf u}_t; \boldsymbol{\theta}_t) + {\bf n}_t
\end{cases},
\end{equation}
where $\boldsymbol{\theta}_t \in \mathbb{R}^{d \times 1}$ are the parameters evaluated at time $t$, $y ({\bf u}_t) = [ 
 y(u_t^1),  y(u_t^2), \ldots, y(u_t^N)  ]^\top \in \mathbb{R}^{N \times 1}$ is the $N$-dimensional observation vector at time $t$, and $h({\bf u}_t; \boldsymbol{\theta}_t) = [
h(u_t^1; \boldsymbol{\theta}_t),
h(u_t^2; \boldsymbol{\theta}_t), 
\ldots,
h(u_t^N; \boldsymbol{\theta}_t)  
^\top \in \mathbb{R}^{N \times 1}$
where $h(u; \boldsymbol{\theta})$ is a nonlinear observation function with input $u$ and parameters $\boldsymbol{\theta}$.

The evolution noise ${\bf v}_t$ is white ($\mathbb{E}[{\bf v}_t] = {\bf 0}$) with covariance ${\bf P}_{{\bf v}_t} \triangleq \mathbb{E}[{\bf v}_t {\bf v}_t^\top]$, $\ \mathbb{E}[{\bf v}_t {\bf v}_{t'}^\top] = {\bf 0}, \quad \forall t \neq t'$.

The observation noise ${\bf n}_t$ is white ($\mathbb{E}[{\bf n}_{t}]={\bf 0}$) with covariance ${\bf P}_{{\bf n}_t} \triangleq \mathbb{E}[{\bf n}_t {\bf n}_t^\top]$, $\ \mathbb{E}[{\bf n}_t {\bf n}_{t'}^\top] = {\bf 0}, \quad \forall t \neq t'$.

The EKF sets the estimation of the parameters $\boldsymbol{\theta}$ at time $t$ according to the conditional expectation: 
\[\boldsymbol{\hat{\theta}}_{t|t}  \triangleq \mathbb{E} [ \boldsymbol{\theta}_t | y_{1:t})] \]
\begin{equation}
\label{suppeq:weight_estimation}
\boldsymbol{\hat{\theta}}_{t|t-1}  \triangleq \mathbb{E} [ \boldsymbol{\theta}_t | y_{1:t-1}]  = \boldsymbol{\hat{\theta}}_{t-1|t-1} 
\end{equation}
 
where with some abuse of notation, $y_{1:t'}$ are the observations gathered up to time $t'$: $y ({\bf u}_1), \ldots, y ({\bf u}_t')$. The {\it parameters errors} are defined by: 
\[\boldsymbol{\tilde{\theta}}_{t|t}  \triangleq \boldsymbol{\theta}_t - \boldsymbol{\hat{\theta}}_{t|t} \]
\begin{equation}
\label{suppeq:weights_error}
\boldsymbol{\tilde{\theta}}_{t|t-1} \triangleq \boldsymbol{\theta}_t - \boldsymbol{\hat{\theta}}_{t|t-1}
\end{equation}

The conditional {\it error covariances} are given by: 
\begin{align*}
{\bf P}_{t|t}  & \triangleq \mathbb{E} \big[  \boldsymbol{\tilde{\theta}}_{t|t}  \boldsymbol{\tilde{\theta}}_{t|t}^\top | y_{1:t}\big], \\
{\bf P}_{t|t-1} & \triangleq \mathbb{E} \big[  \boldsymbol{\tilde{\theta}}_{t|t-1}  \boldsymbol{\tilde{\theta}}_{t|t-1}^\top | y_{1:t-1}\big] \\
& = \mathbb{E} \big[  (\boldsymbol{\theta}_t - \boldsymbol{\hat{\theta}}_{t|t-1}) (\boldsymbol{\theta}_t - \boldsymbol{\hat{\theta}}_{t|t-1})^\top | y_{1:t-1}\big]\\
& = \mathbb{E} \big[  (\boldsymbol{\theta}_{t-1} + {\bf v}_t - \boldsymbol{\hat{\theta}}_{t-1|t-1}) \\
& \quad \quad \quad (\boldsymbol{\theta}_{t-1} + {\bf v}_t - \boldsymbol{\hat{\theta}}_{t-1|t-1})^\top | y_{1:t-1}\big]\\
& = \mathbb{E} \big[  (\boldsymbol{\tilde{\theta}}_{t-1|t-1} + {\bf v}_t) (\boldsymbol{\tilde{\theta}}_{t-1|t-1} + {\bf v}_t)^\top | y_{1:t-1}\big]\\
& = \mathbb{E} \big[  (\boldsymbol{\tilde{\theta}}_{t-1|t-1} \boldsymbol{\tilde{\theta}}_{t-1|t-1}^\top | y_{1:t-1}\big]\\
& +2 \mathbb{E} \big[ \boldsymbol{\tilde{\theta}}_{t-1|t-1} {\bf v}_t^\top| y_{1:t-1} \big] + \mathbb{E} \big[ {\bf v}_t {\bf v}_t^\top | y_{1:t-1}\big]\\
& = {\bf P}_{t-1|t-1} + {\bf P}_{{\bf v}_t}.
\end{align*}
\begin{equation}
\label{suppeq:error_covariance}
\boxed{{\bf P}_{t|t-1} = {\bf P}_{t-1|t-1} + {\bf P}_{{\bf v}_t}}
\end{equation}
EKF considers several statistics of interest at each time step:
{\it The prediction of the observation function}:
\[{\bf \hat{y}}_{t|t-1} \triangleq  \mathbb{E}[h({\bf u}_t; \boldsymbol{\theta}_t)|y_{1:t-1}].\]
{\it The observation innovation}:
\[{\bf \tilde{y}}_{t|t-1} \triangleq h({\bf u}_t; \boldsymbol{\theta}_t) - {\bf \hat{y}}_{t|t-1}.\]
{\it The covariance between the parameters error and the innovation}:
\[{\bf P}_{\boldsymbol{\tilde{\theta}}_t,{\bf \tilde{y}}_{t}}  \triangleq \mathbb{E}[ \boldsymbol{\tilde{\theta}}_{t|t-1} {\bf \tilde{y}}_{t|t-1}^\top |y_{1:t-1}].\]
{\it The covariance of the innovation}:
\[ {\bf P}_{{\bf \tilde{y}}_t}  \triangleq \mathbb{E}[( {\bf \tilde{y}}_{t|t-1} {\bf \tilde{y}}_{t|t-1}^\top|y_{1:t-1}] + {\bf P}_{{\bf n}_t}. \]
The {\it Kalman gain}:
\[{\bf K}_t \triangleq {\bf P}_{\boldsymbol{\tilde{\theta}}_t,{\bf \tilde{y}}_{t}}  {\bf P}_{{\bf \tilde{y}}_t}^{-1}.\]
The above statistics serve for the update of the parameters and the error covariance:
\begin{equation}
\label{suppeq:kalman update}
\begin{cases}
\boldsymbol{\hat{\theta}}_{t|t} = \boldsymbol{\hat{\theta}}_{t|t-1} +    {\bf K}_t \big( y({\bf u}_t) - h({\bf u}_t; \boldsymbol{\hat{\theta}}_{t|t-1}) \big),\\
{\bf P}_{t|t} = {\bf P}_{t|t-1} - {\bf K}_t  {\bf P}_{{\bf \tilde{y}}_t} {\bf K}_t^\top.
\end{cases}
\end{equation}

\subsection{EKF for Value Function Estimation}
\label{suppSection:EKF_Qlearning}
When applying the EKF formulation to value functions approximation, the observation at time $t$ is  the target label $y({\bf u}_t)$ (see Table 1 in the main article), and the observation function $h$ can be the state value function, the state action value function or the advantage function.   

The EKF uses a first order Taylor series linearization for the observation function:
\begin{equation}
\label{suppeq:Linearization}
h({\bf u}_t; \boldsymbol{\theta}_t) 
= h({\bf u}_t; \boldsymbol{\hat{\theta}}) +   \nabla_{\boldsymbol{\theta}_t} h({\bf u}_t; \boldsymbol{\hat{\theta}})^\top \big( \boldsymbol{\theta}_{t} - \boldsymbol{\hat{\theta}} \big),
\end{equation}
where $\nabla_{\boldsymbol{\theta}_t} h({\bf u}_t; \boldsymbol{\hat{\theta}})  = \begin{bmatrix} 
\nabla_{\boldsymbol{\theta}_t} h(u_t^1; \boldsymbol{\hat{\theta}}) , \ldots , \nabla_{\boldsymbol{\theta}_t} h(u_t^N; \boldsymbol{\hat{\theta}})  \end{bmatrix} \in \mathbb{R}^{d \times N}$
and $\boldsymbol{\hat{\theta}}$ is typically chosen to be the previous estimation of the parameters at time $t-1$,  $\boldsymbol{\hat{\theta}}_{t|t-1}$. This linearization helps in computing the statistics of interest. Recall that the expectation here is over the random variable $\boldsymbol{\theta}_t$ where $\boldsymbol{\hat{\theta}}_{t|t-1}$ is fixed. For simplicity, we keep to write $\boldsymbol{\hat{\theta}}$. 
{\it The prediction of the observation function}:
\begin{align*}
{\bf \hat{y}}_{t|t-1} & \triangleq  \mathbb{E}[h({\bf u}_t; \boldsymbol{\theta}_t)|y_{1:t-1}]\\
& \underbrace{=}_{(\ref{suppeq:Linearization})} \mathbb{E}\Big[h({\bf u}_t,\boldsymbol{\hat{\theta}})
+   \nabla_{\boldsymbol{\theta}_t} h({\bf u}_t; \boldsymbol{\hat{\theta}})^\top \big( \boldsymbol{\theta}_{t} - \boldsymbol{\hat{\theta}} \big) |y_{1:t-1}\Big]\\
& = h({\bf u}_t; \boldsymbol{\hat{\theta}})  +   \nabla_{\boldsymbol{\theta}_t} h({\bf u}_t; \boldsymbol{\hat{\theta}})^\top \big( \mathbb{E}[  \boldsymbol{\theta}_{t} |y_{1:t-1}] - \boldsymbol{\hat{\theta}} \big)\\
& \underbrace{=}_{(\ref{suppeq:weight_estimation})} h({\bf u}_t; \boldsymbol{\hat{\theta}})  +    \nabla_{\boldsymbol{\theta}_t} h({\bf u}_t; \boldsymbol{\hat{\theta}})^\top \big( \boldsymbol{\hat{\theta}} - \boldsymbol{\hat{\theta}} \big)\\
& = h({\bf u}_t; \boldsymbol{\hat{\theta}}) = h({\bf u}_t; \boldsymbol{\hat{\theta}}_{t|t-1})
\end{align*}
We conclude that:
\begin{equation}
\label{suppeq:prediction_observation}
\boxed{{\bf \hat{y}}_{t|t-1}   = h({\bf u}_t; \boldsymbol{\hat{\theta}}_{t|t-1}) }
\end{equation}
{\it The observation innovation}:
\begin{equation}
\label{suppeq:observation_innovation}
\boxed{ {\bf \tilde{y}}_{t|t-1} \triangleq h({\bf u}_t; \boldsymbol{\theta}_t) -{\bf \hat{y}}_{t|t-1}  \underbrace{=}_{(\ref{suppeq:prediction_observation})} h({\bf u}_t,  \boldsymbol{\theta}_t) - h({\bf u}_t; \boldsymbol{\hat{\theta}}_{t|t-1}) }
\end{equation} 
Let's simplify the following:
\begin{align}
\label{suppeq:simplify}
\nonumber & h({\bf u}_t,  \boldsymbol{\theta}_t) - h({\bf u}_t; \boldsymbol{\hat{\theta}}) \\
\nonumber & \underbrace{=}_{(\ref{suppeq:Linearization})} \big( \cancel{h({\bf u}_t; \boldsymbol{\hat{\theta}})}  +     \nabla_{\boldsymbol{\theta}_t} h({\bf u}_t; \boldsymbol{\hat{\theta}})^\top \big( \boldsymbol{\theta}_{t} - \boldsymbol{\hat{\theta}} \big) - \cancel{h({\bf u}_t; \boldsymbol{\hat{\theta}}) \big)}\\
& = \nabla_{\boldsymbol{\theta}_t} h({\bf u}_t; \boldsymbol{\hat{\theta}})^\top \big( \boldsymbol{\theta}_{t} - \boldsymbol{\hat{\theta}} \big)
\end{align}
{\it The covariance between the parameters error and the innovation} (here we also denote $\boldsymbol{\hat{\theta}} = \boldsymbol{\hat{\theta}}_{t|t-1}$):
\begin{align*}
{\bf P}_{\boldsymbol{\tilde{\theta}}_t,{\bf \tilde{y}}_{t}}  & \triangleq \mathbb{E}[ \boldsymbol{\tilde{\theta}}_{t|t-1} {\bf \tilde{y}}_{t|t-1}^\top |y_{1:t-1}]\\
& \underbrace{=}_{(\ref{suppeq:weights_error}) + (\ref{suppeq:observation_innovation})} \mathbb{E}[ \big(\boldsymbol{\theta}_t - \boldsymbol{\hat{\theta}} \big) \big( h({\bf u}_t; \boldsymbol{\theta}_t) - h({\bf u}_t; \boldsymbol{\hat{\theta}}) \big)^\top |y_{1:t-1}]\\
& \underbrace{=}_{(\ref{suppeq:simplify})} \mathbb{E}[ \big(\boldsymbol{\theta}_t - \boldsymbol{\hat{\theta}} \big)  \big( \boldsymbol{\theta}_{t} - \boldsymbol{\hat{\theta}} \big)^\top \nabla_{\boldsymbol{\theta}_t} h({\bf u}_t; \boldsymbol{\hat{\theta}})  |y_{1:t-1}]\\
& \underbrace{=}_{(\ref{suppeq:weights_error})} \mathbb{E}[ \boldsymbol{\tilde{\theta}}_{t|t-1}  \boldsymbol{\tilde{\theta}}_{t|t-1}^\top  |y_{1:t-1}] \nabla_{\boldsymbol{\theta}_t} h({\bf u}_t; \boldsymbol{\hat{\theta}})\\
& \underbrace{=}_{(\ref{suppeq:error_covariance})} {\bf P}_{t|t-1} \nabla_{\boldsymbol{\theta}_t} h({\bf u}_t; \boldsymbol{\hat{\theta}}_{t|t-1})
\end{align*}
\begin{equation}
\label{suppeq:covariance_weights_innovation}
\boxed{{\bf P}_{\boldsymbol{\tilde{\theta}}_t,{\bf \tilde{y}}_{t}} = {\bf P}_{t|t-1} \nabla_{\boldsymbol{\theta}_t} h({\bf u}_t; \boldsymbol{\hat{\theta}}_{t|t-1})}
\end{equation}
{\it The covariance of the innovation}:
\begin{align*}
 {\bf P}_{{\bf \tilde{y}}_t}  & \triangleq \mathbb{E}[( {\bf \tilde{y}}_{t|t-1} {\bf \tilde{y}}_{t|t-1}^\top|y_{1:t-1}] + {\bf P}_{{\bf n}_t}\\
 & \underbrace{=}_{(\ref{suppeq:observation_innovation})} \mathbb{E}[ \big( h({\bf u}_t,  \boldsymbol{\theta}_t) - h({\bf u}_t; \boldsymbol{\hat{\theta}}) \big) \big( h({\bf u}_t,  \boldsymbol{\theta}_t) - h({\bf u}_t; \boldsymbol{\hat{\theta}}) \big)^\top |y_{1:t-1}]\\
 & \quad \quad + {\bf P}_{{\bf n}_t}\\
 & \underbrace{=}_{(\ref{suppeq:simplify})} \mathbb{E} \Big[     \nabla_{\boldsymbol{\theta}_t} h({\bf u}_t; \boldsymbol{\hat{\theta}})^\top \big( \boldsymbol{\theta}_{t} - \boldsymbol{\hat{\theta}} \big)  \big( \boldsymbol{\theta}_{t} - \boldsymbol{\hat{\theta}} \big)^\top   \nabla_{\boldsymbol{\theta}_t} h({\bf u}_t; \boldsymbol{\hat{\theta}})  |y_{1:t-1} \Big]\\
 & \quad \quad + {\bf P}_{{\bf n}_t}\\
 & \underbrace{=}_{(\ref{suppeq:weights_error})} \nabla_{\boldsymbol{\theta}_t} h({\bf u}_t,   \boldsymbol{\hat{\theta}})^\top \mathbb{E}[ \boldsymbol{\tilde{\theta}}_{t|t-1}  \boldsymbol{\tilde{\theta}}_{t|t-1}^\top    |y_{1:t-1}] \\
 &\quad \quad \nabla_{\boldsymbol{\theta}_t} h({\bf u}_t,  \boldsymbol{\hat{\theta}}) + {\bf P}_{{\bf n}_t}\\
 & =\nabla_{\boldsymbol{\theta}_t} h({\bf u}_t,   \boldsymbol{\hat{\theta}})^\top {\bf P}_{t|t-1} \nabla_{\boldsymbol{\theta}_t} h({\bf u}_t,   \boldsymbol{\hat{\theta}})^\top + {\bf P}_{{\bf n}_t}
\end{align*}
\begin{equation}
\label{suppeq:covariance_innovation}
\boxed{ {\bf P}_{{\bf \tilde{y}}_t}   = \nabla_{\boldsymbol{\theta}_t} h({\bf u}_t; \boldsymbol{\hat{\theta}}_{t|t-1})^\top {\bf P}_{t|t-1} \nabla_{\boldsymbol{\theta}_t} h({\bf u}_t; \boldsymbol{\hat{\theta}}_{t|t-1}) + {\bf P}_{{\bf n}_t}}
\end{equation}
The {\it Kalman gain}:
\begin{align}
\label{suppeq:kalman_gain}
\nonumber {\bf K}_t & \triangleq {\bf P}_{\boldsymbol{\tilde{\theta}}_t,{\bf \tilde{y}}_{t}}  {\bf P}_{{\bf \tilde{y}}_t}^{-1}\\
\nonumber & \underbrace{=}_{(\ref{suppeq:covariance_weights_innovation}) + (\ref{suppeq:covariance_innovation})} {\bf P}_{t|t-1} \nabla_{\boldsymbol{\theta}_t} h({\bf u}_t,  \boldsymbol{\hat{\theta}}_{t|t-1})\\
& \Big( \nabla_{\boldsymbol{\theta}_t} h({\bf u}_t,  \boldsymbol{\hat{\theta}}_{t|t-1})^\top {\bf P}_{t|t-1} \nabla_{\boldsymbol{\theta}_t} h({\bf u}_t,  \boldsymbol{\hat{\theta}}_{t|t-1})  + {\bf P}_{{\bf n}_t} \Big)^{-1}
\end{align}

and the update for the parameters of the  value function and the error covariance are the same as in Equation (\ref{suppeq:kalman update}) as we prove in Theorem 1.
 
\subsection{A Bayesian approach: MAP estimator}
We adopt the Bayesian approach in which we are interested in finding the optimal set of parameters $\boldsymbol{\theta}_t$ that maximizes the posterior distribution of the parameters given the observations we have gathered up to time $t$, denoted as the  $y_{1:t}$.

According to Bayes rule, the posterior distribution is defined as: 
\[p(\boldsymbol{\theta}_t|y_{1:t}) = \frac{p(y_{1:t}|\boldsymbol{\theta}_t) p(\boldsymbol{\theta}_t)}{p(y_{1:t})}\]
where $p(y_{1:t}|\boldsymbol{\theta})$ is the {\it likelihood} of the observations given the parameters $\boldsymbol{\theta}$ and $p(\boldsymbol{\theta})$ is the {\it prior} distribution over $\boldsymbol{\theta}$.	We will expend the term of the posterior \cite{van2004sigma}:
\begin{align}
\nonumber p(\boldsymbol{\theta}_t|y_{1:t}) & = \frac{p(y_{1:t}|\boldsymbol{\theta}_t)p(\boldsymbol{\theta}_t)}{p(y_{1:t})} \\
& = \frac{p(y_t|y_{1:t-1},\boldsymbol{\theta}_t) p(y_{1:t-1}|\boldsymbol{\theta}_t)p(\boldsymbol{\theta}_t)}{p(y_{1:t})} \label{eq:PosteriorWithNoise1} \\ 
& = \frac{p(y_t|\boldsymbol{\theta}_t) p(y_{1:t-1}|\boldsymbol{\theta}_t)p(\boldsymbol{\theta}_t)  }{p(y_{1:t})} \cdot \frac{p(y_{1:t-1})}{p(y_{1:t-1})}  \label{eq:PosteriorWithNoise2}\\ 
& = \frac{p(y_t|\boldsymbol{\theta}_t)p(\boldsymbol{\theta}_t|y_{1:t-1}) p(y_{1:t-1}) }{p(y_{1:t})} \label{eq:PosteriorWithNoise3}
\end{align}
The transition in (\ref{eq:PosteriorWithNoise1}) is according to the conditional probability:
\begin{align*}
p(y_{1:t}|\boldsymbol{\theta}_t) & = p(y_t, y_{1:t-1}|\boldsymbol{\theta}_t)\\
& = \frac{p(y_t, y_{1:t-1},\boldsymbol{\theta}_t)}{p(\boldsymbol{\theta}_t)} \\
& = \frac{p(y_{1:t-1},\boldsymbol{\theta}_t)p(y_t|y_{1:t-1},\boldsymbol{\theta}_t)}{p(\boldsymbol{\theta}_t)} \\
& = p(y_{1:t-1}|\boldsymbol{\theta}_t)p(y_t|y_{1:t-1},\boldsymbol{\theta}_t) 
\end{align*}
The transition in (\ref{eq:PosteriorWithNoise2}) is according to the conditional independence: $p(y_t|y_{1:t-1},\boldsymbol{\theta}_t) = p(y_t|\boldsymbol{\theta}_t)$, and we multiplied the numerator and the dominator by $p(y_{1:t-1})$.\\
The transition in (\ref{eq:PosteriorWithNoise3}) is according to Bayes rule: $p(\boldsymbol{\theta}_t|y_{1:t-1}) = \frac{p(y_{1:t-1}|\boldsymbol{\theta}_t)p(\boldsymbol{\theta}_t)}{p(y_{1:t-1})}$.

The MAP estimator for $\boldsymbol{\theta}_t$ is the one who maximizes the posterior distribution described in (\ref{eq:PosteriorWithNoise3}).
\begin{align}
\nonumber \boldsymbol{\theta}_{t}^{MAP}  = & \arg\max_{\boldsymbol{\theta}_t} && \big\{ p(\boldsymbol{\theta}_t|y_{1:t}) \big\}\\
\nonumber = &\arg\max_{\boldsymbol{\theta}_t} && \Big\{  \frac{p(y_t|\boldsymbol{\theta}_t)p(\boldsymbol{\theta}_t|y_{1:t-1}) p(y_{1:t-1}) }{p(y_{1:t})} \Big\}\\
\nonumber = &\arg\max_{\boldsymbol{\theta}_t} && \big\{  p(y_t|\boldsymbol{\theta}_t)p(\boldsymbol{\theta}_t|y_{1:t-1})  \big\}\\
\nonumber = &\arg\max_{\boldsymbol{\theta}_t} && \big\{ \log \big(  p(y_t|\boldsymbol{\theta}_t)p(\boldsymbol{\theta}_t|y_{1:t-1})  \big) \big\} \\
\nonumber = &\arg\max_{\boldsymbol{\theta}_t} && \big\{ \log   p(y_t|\boldsymbol{\theta}_t) + \log p(\boldsymbol{\theta}_t|y_{1:t-1})   \big\} \\
= &\arg\min_{\boldsymbol{\theta}_t} && \big\{ - \log   p(y_t|\boldsymbol{\theta}_t) - \log p(\boldsymbol{\theta}_t|y_{1:t-1})   \big\} \label{eq:MAPln}
\end{align}
In (\ref{eq:MAPln}) We used the derivation in ($\ref{eq:PosteriorWithNoise3}$) and the fact that the argument which maximizes the posterior is the same as the argument that maximizes the $\log (\cdot)$ of the posterior. In addition this argument also minimizes the negative $\log(\cdot)$. 

We will replace here $y_t = y({\bf u}_t)$ and receive:
\begin{equation}
\label{eq:MAPln2}
\boldsymbol{\theta}_{t}^{MAP} = \arg\min_{\boldsymbol{\theta}_t}  \big\{ - \log   p(y({\bf u}_t)|\boldsymbol{\theta}_t) - \log p(\boldsymbol{\theta}_t|y_{1:t-1})   \big\}
\end{equation}

In order to solve (\ref{eq:MAPln2}), we consider the EKF formulation for the value function parameters.

\subsection{Gaussian assumptions}
When estimating using the EKF, it is common to make the following assumptions regarding the likelihood and the posterior in Equation (\ref{eq:MAPln2}):
\begin{assumption}
	\label{suppAs:ConditionalIndependance}
	The likelihood $p(y({\bf u}_t)|\boldsymbol{\theta}_t)$ is assumed to be Gaussian: 
	$y({\bf u}_t)|\boldsymbol{\theta}_t \sim \mathcal{N}(h({\bf u}_t,  \boldsymbol{\theta}_t), {\bf P}_{{\bf n}_t})$. 
\end{assumption}

\begin{assumption}
	\label{suppAs:GaussianPosterior}
	The posterior distribution $p(\boldsymbol{\theta}_t|y_{1:t-1})$ is assumed to be Gaussian: $\boldsymbol{\theta}_t|y_{1:t-1} \sim \mathcal{N}(\boldsymbol{\hat{\theta}}_{t|t-1},{\bf P}_{t|t-1})$.
\end{assumption}
Following are the calculations for the means and covariances in Assumptions \ref{suppAs:ConditionalIndependance} and \ref{suppAs:GaussianPosterior}. For the likelihood $p(y({\bf u}_t)|\boldsymbol{\theta}_t)$:
\begin{align}
\label{suppeq:expected_y}
\nonumber \mathbb{E} \big[y(o_t)|\boldsymbol{\theta}_t \big] & \underbrace{=}_{(\ref{suppeq:Extended-Kalman})} \mathbb{E} \big[h({\bf u}_t; \boldsymbol{\theta}_t) + {\bf n}_t|\boldsymbol{\theta}_t \big]\\
\nonumber &  = \mathbb{E} \big[h({\bf u}_t; \boldsymbol{\theta}_t) |\boldsymbol{\theta}_t \big] + \underbrace{\mathbb{E} \big[ {\bf n}_t |\boldsymbol{\theta}_t\big]}_{={\bf 0}}\\
& = h({\bf u}_t; \boldsymbol{\theta}_t)
\end{align}
Let's evaluate the following:
\begin{align}
\label{suppeq:diff_y}
\nonumber y({\bf u}_t) - \mathbb{E} \big[y({\bf u}_t)|\boldsymbol{\theta}_t \big]  & \underbrace{=}_{(\ref{suppeq:Extended-Kalman}) + (\ref{suppeq:expected_y})} h({\bf u}_t; \boldsymbol{\theta}_t) + {\bf n}_t - h({\bf u}_t; \boldsymbol{\theta}_t)\\
& = {\bf n}_t
\end{align}
\begin{align*}
Cov(y({\bf u}_t)|\boldsymbol{\theta}_t) & \triangleq 
\mathbb{E} \big[\big(y({\bf u}_t) - \mathbb{E} \big[y({\bf u}_t)|\boldsymbol{\theta}_t \big] \big)\\
& \quad \big(y({\bf u}_t) - \mathbb{E} \big[y({\bf u}_t)|\boldsymbol{\theta}_t \big] \big)^\top|\boldsymbol{\theta}_t \big] \\
& \underbrace{=}_{(\ref{suppeq:diff_y})} \mathbb{E} \big[  {\bf n}_t {\bf n}_t^\top  |\boldsymbol{\theta}_t \big]\\
& = {\bf P}_{{\bf n}_t}
\end{align*}
For the posterior $p(\boldsymbol{\theta}_t|y_{1:t-1})$:  
$\mathbb{E}_{\boldsymbol{\theta}_t} \big[ \boldsymbol{\theta}_t|y_{1:t-1}\big] \underbrace{=}_{(\ref{suppeq:weight_estimation})} \boldsymbol{\hat{\theta}}_{t|t-1}$.
\begin{align*}
Cov(\boldsymbol{\theta}_t|y_{1:t-1}) \triangleq  
& \mathbb{E}_{\boldsymbol{\theta}_t} \big[ \big(\boldsymbol{\theta}_t - \boldsymbol{\hat{\theta}}_{t|t-1} \big) \big(\boldsymbol{\theta}_t - \boldsymbol{\hat{\theta}}_{t|t-1} \big)^\top |y_{1:t-1} \big]\\
& = \mathbb{E}_{\boldsymbol{\theta}_t} \big[ \boldsymbol{\tilde{\theta}}_{t|t-1} \boldsymbol{\tilde{\theta}}_{t|t-1}^\top |y_{1:t-1} \big] \underbrace{=}_{(\ref{suppeq:error_covariance})} {\bf P}_{t|t-1}
\end{align*}

\subsection{Proof of Theorem 1}
Based on the Gaussian assumptions, we can derive the following Theorem:
\begin{theorem}
	\label{supp:theorem1}
	Under Assumptions \ref{suppAs:ConditionalIndependance} and \ref{suppAs:GaussianPosterior}, $\boldsymbol{\hat{\theta}}^{\text{EKF}}_{t|t}$  (\ref{suppeq:kalman update}) minimizes at each time step $t$ the following regularized objective function:
	\begin{align}
	\label{suppeq:objective_EKF}
	\nonumber L^{\text{EKF}}_t(\boldsymbol{\theta}_t) & =  \frac{1}{2}  \big(\delta ({\bf u}_t; \boldsymbol{\theta}_t) \big)^\top {\bf P}_{{\bf n}_t}^{-1} \big(\delta ({\bf u}_t; \boldsymbol{\theta}_t) \big) \\
	& +  \frac{1}{2}(\boldsymbol{\theta}_t - \boldsymbol{\hat{\theta}}_{t|t-1})^\top {\bf P}_{t|t-1}^{-1} (\boldsymbol{\theta}_t - \boldsymbol{\hat{\theta}}_{t|t-1}),
	\end{align}
	where $\boldsymbol{\hat{\theta}}^{\text{EKF}}_{t|t} \in \arg\min_{\boldsymbol{\theta}_t} L^{\text{EKF}}_t(\boldsymbol{\theta}_t)$.
\end{theorem}

\begin{proof}
	We solve the minimization problem in (\ref{eq:MAPln2}) by substituting the Gaussian Assumptions \ref{suppAs:ConditionalIndependance} and \ref{suppAs:GaussianPosterior}. We show that this minimization problem is equivalent to minimize the objective function $L_t^{\text{EKF}}$ in Theorem \ref{supp:theorem1}.
	{ \small
	\begin{align*}
	\nonumber & \boldsymbol{\hat{\theta}}_{t|t}^{\text{MAP}}  = \arg\min_{\boldsymbol{\theta}_t}  \big\{  - \log \Big( p(y({\bf u}_t)|\boldsymbol{\theta}_t) \Big) - \log \Big( p(\boldsymbol{\theta}_t|y_{1:t-1}) \Big) \big\}\\
	\nonumber & = \arg\min_{\boldsymbol{\theta}_t}  \Big\{ -\log \bigg( \frac{1}{(2\pi)^{N/2} |{\bf P}_{{\bf n}_t}|^{1/2} } \\
	\nonumber & \quad \exp \Big( - \frac{1}{2} \big( y({\bf u}_t) - h({\bf u}_t; \boldsymbol{\theta}_t) \big)^\top  {\bf P}_{{\bf n}_t}^{-1} \big( y({\bf u}_t) - h({\bf u}_t; \boldsymbol{\theta}_t) \big) \Big) \bigg) \\
	\nonumber& -\log \bigg( \frac{1}{(2 \pi)^{d/2} |{\bf P}_{t|t-1}|^{1/2}}  \\
	\nonumber & \quad \exp \Big( - \frac{1}{2} (\boldsymbol{\theta}_t - \boldsymbol{\hat{\theta}}_{t|t-1})^\top  {\bf P}_{t|t-1}^{-1} (\boldsymbol{\theta}_t -  \boldsymbol{\hat{\theta}}_{t|t-1})  \Big) \bigg)  \\
	\nonumber & = \arg\min_{\boldsymbol{\theta}_i}  \Big\{ \frac{1}{2}  \big( y({\bf u}_t) - h({\bf u}_t; \boldsymbol{\theta}_t) \big)^\top  {\bf P}_{{\bf n}_t}^{-1} \big( y({\bf u}_t) - h({\bf u}_t; \boldsymbol{\theta}_t) \big) \\
	\nonumber & \quad  -\log \Big( \frac{1}{(2\pi)^{N/2} |{\bf P}_{{\bf n}_t}|^{1/2} } \Big)  + \frac{1}{2} (\boldsymbol{\theta}_t - \boldsymbol{\hat{\theta}}_{t|t-1})^\top  {\bf P}_{t|t-1}^{-1} (\boldsymbol{\theta}_t -  \boldsymbol{\hat{\theta}}_{t|t-1})\\
	\nonumber & \quad  -\log \Big( \frac{1}{(2 \pi)^{d/2} |{\bf P}_{t|t-1}|^{1/2}}  \Big) \Big\} \\
	\nonumber & = \arg\min_{\boldsymbol{\theta}_i}  \Big\{ \frac{1}{2} \big( y({\bf u}_t) - h({\bf u}_t; \boldsymbol{\theta}_t) \big)^\top  {\bf P}_{{\bf n}_t}^{-1} \big( y({\bf u}_t) - h({\bf u}_t; \boldsymbol{\theta}_t) \big) \\
	\nonumber & + \frac{1}{2} (\boldsymbol{\theta}_t - \boldsymbol{\hat{\theta}}_{t|t-1})^\top  {\bf P}_{t|t-1}^{-1} (\boldsymbol{\theta}_t -  \boldsymbol{\hat{\theta}}_{t|t-1}) \Big\} 
	\end{align*}
	}
	where $|\cdot|$ denotes the determinant. We receive the following objective function:
	\begin{align}
	\label{eq:Objective}
	\nonumber L_t(\boldsymbol{\theta}_t) & =  \frac{1}{2} \big( y({\bf u}_t) - h({\bf u}_t; \boldsymbol{\theta}_t) \big)^\top  {\bf P}_{{\bf n}_t}^{-1} \big( y({\bf u}_t) - h({\bf u}_t; \boldsymbol{\theta}_t) \big) \\
	& + \frac{1}{2} (\boldsymbol{\theta}_t - \boldsymbol{\hat{\theta}}_{t|t-1})^\top  {\bf P}_{t|t-1}^{-1} (\boldsymbol{\theta}_t -  \boldsymbol{\hat{\theta}}_{t|t-1})  
	\end{align}
	
	Which is exactly the objective function (\ref{suppeq:objective_EKF}) in Theorem \ref{supp:theorem1}, with: $\delta ({\bf u}_t; \boldsymbol{\theta}_t) = y({\bf u}_t) - h({\bf u}_t; \boldsymbol{\theta}_t)$. To minimize this objective function we take the derivative of $L^{\text{EKF}}_t(\boldsymbol{\theta}_t)$ with respect to $\boldsymbol{\theta}_t$:
	\begin{align*}
	\nabla_{\boldsymbol{\theta}_t} L^{\text{EKF}}_t(\boldsymbol{\theta}_t) & = -     \nabla_{\boldsymbol{\theta}_t}  h({\bf u}_t,  \boldsymbol{\theta}_t) {\bf P}_{{\bf n}_t}^{-1} \big( y({\bf u}_t) - h({\bf u}_t; \boldsymbol{\theta}_t) \big) \\
	& \quad + {\bf P}_{t|t-1}^{-1} (\boldsymbol{\theta}_t - \boldsymbol{\hat{\theta}}_{t|t-1})  = 0
	\end{align*}
	We use the linearization of the value function in Equation (\ref{suppeq:Linearization}):
	\begin{align*}
	& {\bf P}_{t|t-1}^{-1} (\boldsymbol{\theta}_t - \boldsymbol{\hat{\theta}})   =  \nabla_{\boldsymbol{\theta}_t}  \big( h({\bf u}_t; \boldsymbol{\hat{\theta}}) +   \nabla_{\boldsymbol{\theta}_t} h({\bf u}_t; \boldsymbol{\hat{\theta}})^\top \big( \boldsymbol{\theta}_{t} - \boldsymbol{\hat{\theta}} \big)\big)\\
	& {\bf P}_{{\bf n}_t}^{-1} \Big( y({\bf u}_t) -  h({\bf u}_t; \boldsymbol{\hat{\theta}}) -   \nabla_{\boldsymbol{\theta}_t} h({\bf u}_t; \boldsymbol{\hat{\theta}})^\top \big( \boldsymbol{\theta}_{t} - \boldsymbol{\hat{\theta}} \big)  \Big)    \\
	& =  \nabla_{\boldsymbol{\theta}_t} h({\bf u}_t; \boldsymbol{\hat{\theta}}) {\bf P}_{{\bf n}_t}^{-1} \Big( y({\bf u}_t) -  h({\bf u}_t; \boldsymbol{\hat{\theta}}) \Big)    \\
	& - \nabla_{\boldsymbol{\theta}_t} h({\bf u}_t; \boldsymbol{\hat{\theta}}) {\bf P}_{{\bf n}_t}^{-1} \nabla_{\boldsymbol{\theta}_t} h({\bf u}_t; \boldsymbol{\hat{\theta}})^\top  \big( \boldsymbol{\theta}_{t} - \boldsymbol{\hat{\theta}} \big)
	 \end{align*}
	 We receive that:
	 \begin{align*}
	 & \Big( {\bf P}_{t|t-1}^{-1} + \nabla_{\boldsymbol{\theta}_t} h({\bf u}_t; \boldsymbol{\hat{\theta}}) {\bf P}_{{\bf n}_t}^{-1} \nabla_{\boldsymbol{\theta}_t} h({\bf u}_t; \boldsymbol{\hat{\theta}})^\top\Big)  (\boldsymbol{\theta}_t - \boldsymbol{\hat{\theta}})\\
	 & =  \nabla_{\boldsymbol{\theta}_t} h({\bf u}_t; \boldsymbol{\hat{\theta}}) {\bf P}_{{\bf n}_t}^{-1} \Big( y({\bf u}_t) -  h({\bf u}_t; \boldsymbol{\hat{\theta}}) \Big)
	 \end{align*}
	 and finally:
	  \begin{align}
	  \label{suppeq:weight_MAP}
	  \nonumber \boldsymbol{\theta}_t & = \boldsymbol{\hat{\theta}} +  \Big( {\bf P}_{t|t-1}^{-1} + \nabla_{\boldsymbol{\theta}_t} h({\bf u}_t; \boldsymbol{\hat{\theta}}) {\bf P}_{{\bf n}_t}^{-1} \nabla_{\boldsymbol{\theta}_t} h({\bf u}_t; \boldsymbol{\hat{\theta}})^\top\Big)^{-1}  \\
	  &  \nabla_{\boldsymbol{\theta}_t} h({\bf u}_t; \boldsymbol{\hat{\theta}}) {\bf P}_{{\bf n}_t}^{-1} \Big( y({\bf u}_t) -  h({\bf u}_t; \boldsymbol{\hat{\theta}}) \Big)
	  \end{align}
	  For simplicity we denote:
	  $\nabla {\bf h} = \nabla_{\boldsymbol{\theta}_t} h({\bf u}_t; \boldsymbol{\hat{\theta}})$. 
	We will now simplify the following term:
	{\small 
	\begin{align}
	\label{suppeq:Kt_derivation}
	\nonumber & \Big( {\bf P}_{t|t-1}^{-1} + \nabla {\bf h} {\bf P}_{{\bf n}_t}^{-1} \nabla {\bf h}^\top\Big)^{-1} \nabla {\bf h} {\bf P}_{{\bf n}_t}^{-1}\\
	\nonumber & = \Big( {\bf P}_{t|t-1}^{-1} + \nabla {\bf h} {\bf P}_{{\bf n}_t}^{-1} \nabla {\bf h}^\top\Big)^{-1} \nabla {\bf h} {\bf P}_{{\bf n}_t}^{-1} \Big( \nabla {\bf h}^\top {\bf P}_{t|t-1} \nabla {\bf h} + {\bf P}_{{\bf n}_t}\Big)\\
	\nonumber & \Big( \nabla {\bf h}^\top {\bf P}_{t|t-1} \nabla {\bf h} + {\bf P}_{{\bf n}_t}\Big)^{-1}\\
	\nonumber & = \Big( {\bf P}_{t|t-1}^{-1} + \nabla {\bf h} {\bf P}_{{\bf n}_t}^{-1} \nabla {\bf h}^\top\Big)^{-1}  \Big( \nabla {\bf h} {\bf P}_{{\bf n}_t}^{-1} \nabla {\bf h}^\top {\bf P}_{t|t-1} \nabla {\bf h} \\
	\nonumber & + \nabla {\bf h} {\bf P}_{{\bf n}_t}^{-1} {\bf P}_{{\bf n}_t}\Big) \Big( \nabla {\bf h}^\top {\bf P}_{t|t-1} \nabla {\bf h} + {\bf P}_{{\bf n}_t}\Big)^{-1}\\
	\nonumber & = \Big( {\bf P}_{t|t-1}^{-1} + \nabla {\bf h} {\bf P}_{{\bf n}_t}^{-1} \nabla {\bf h}^\top\Big)^{-1}  \Big( \nabla {\bf h} {\bf P}_{{\bf n}_t}^{-1} \nabla {\bf h}^\top  \\
	\nonumber & +  {\bf P}_{t|t-1}^{-1} \Big) {\bf P}_{t|t-1} \nabla {\bf h} \Big( \nabla {\bf h}^\top {\bf P}_{t|t-1} \nabla {\bf h} + {\bf P}_{{\bf n}_t}\Big)^{-1}\\
	\nonumber & = {\bf P}_{t|t-1} \nabla {\bf h} \Big( \nabla {\bf h}^\top {\bf P}_{t|t-1} \nabla {\bf h} + {\bf P}_{{\bf n}_t}\Big)^{-1}\\
	\nonumber & \underbrace{=}_{(\ref{suppeq:covariance_weights_innovation})+(\ref{suppeq:covariance_innovation})} {\bf P}_{\boldsymbol{\tilde{\theta}}_t,{\bf \tilde{y}}_{t}}  {\bf P}_{{\bf \tilde{y}}_t}^{-1}\\
	& \underbrace{=}_{(\ref{suppeq:kalman_gain})} {\bf K}_t
	\end{align}
	}
	Substituting this results in Equation (\ref{suppeq:weight_MAP}), we receive the EKF update for the parameters:
	\begin{align}
	\label{suppeq:EKF_weight}
	\boldsymbol{\hat{\theta}}^{\text{EKF}}_{t|t} & = \boldsymbol{\hat{\theta}}_{t|t-1} +  {\bf K}_t \big( y({\bf u}_t) - h({\bf u}_t; \boldsymbol{\hat{\theta}}_{t|t-1}) \big) 
	\end{align}
	which is exactly as in Equation (\ref{suppeq:kalman update}).
	
	We will now develop the term $\Big( {\bf P}_{t|t-1}^{-1} +\nabla {\bf h} {\bf P}_{{\bf n}_t}^{-1} \nabla {\bf h}^\top\Big)^{-1}$ that appears in (\ref{suppeq:weight_MAP}) by using the matrix inversion lemma:
	\begin{equation}
	\label{suppeq:MatrixInversionLemma}
	({\bf B}^{-1} + {\bf C}{\bf D}^{-1}{\bf C}^\top)^{-1} = {\bf B} - {\bf BC}({\bf D}+ {\bf C}^\top {\bf BC})^{-1} {\bf C}^\top{\bf B}
	\end{equation}
	where ${\bf B}$ is a square symmetric positive-definite (and hence invertible) matrix. For this purpose we assume that the error covariance matrix of $\boldsymbol{\theta}_t$, ${\bf P}_{t|t-1}$, is symmetric and positive-definite.
	{\small 
	\begin{align}
	\nonumber & \Big( {\bf P}_{t|t-1}^{-1} + \nabla {\bf h} {\bf P}_{{\bf n}_t}^{-1} \nabla {\bf h}^\top\Big)^{-1} \\
	\nonumber & \underbrace{=}_{(\ref{suppeq:MatrixInversionLemma})} {\bf P}_{t|t-1} - {\bf P}_{t|t-1} \nabla {\bf h}({\bf P}_{{\bf n}_t} + \nabla {\bf h}^\top {\bf P}_{t|t-1} \nabla {\bf h})^{-1}\nabla {\bf h}^\top {\bf P}_{t|t-1} \\
	\nonumber & \underbrace{=}_{(\ref{suppeq:Kt_derivation})} {\bf P}_{t|t-1} - {\bf K}_t \nabla {\bf h}^\top {\bf P}_{t|t-1}\\
	\nonumber & \underbrace{=}_{(\ref{suppeq:covariance_weights_innovation})} {\bf P}_{t|t-1} - {\bf K}_t {\bf P}_{\boldsymbol{\tilde{\theta}}_t,{\bf \tilde{y}}_{t}}^\top\\
	\nonumber & \underbrace{=}_{(\ref{suppeq:kalman_gain})} {\bf P}_{t|t-1} - {\bf K}_t  {\bf P}_{{\bf \tilde{y}}_t} {\bf K}_t^\top
	\end{align} 
	}
	
	We can write the update of the parameters error covariance as:
	\begin{equation}
	\label{suppeq:error_covariance_update}
	\boxed { {\bf P}_{t|t} = {\bf P}_{t|t-1} - {\bf K}_t  {\bf P}_{{\bf \tilde{y}}_t} {\bf K}_t^\top }
	\end{equation}
	
	We conclude the proof by stating that the optimal parameter $\boldsymbol{\hat{\theta}}_{t|t}^{\text{EKF}}$ in (\ref{suppeq:kalman update}) is the solution to the minimization of the objective function in (\ref{suppeq:objective_EKF}):
	\[\boldsymbol{\hat{\theta}}_{t|t}^{\text{EKF}} \in \arg\min_{\boldsymbol{\theta}_{t}} L_t^{\text{EKF}}(\boldsymbol{\theta}_t)\] 
\end{proof}

\subsection{Proof of Colloraly 1}
\begin{proof}
	If ${\bf P}_{{\bf n}_t}$ is diagonal with diagonal elements $\sigma_i=N$, where $N$ is the number of samples in a batch, then:
	\begin{align*}
	\frac{1}{2}  \delta({\bf u}_t; \boldsymbol{\theta}_t)^\top {\bf P}_{{\bf n}_t}^{-1}  \delta({\bf u}_t; \boldsymbol{\theta}_t) & = \frac{1}{2 N} \sum_{i=1}^N  \delta^2(u_t^i,\boldsymbol{\theta}_t)\\
	& = L^{\text{MLE}}_t (\boldsymbol{\theta}_t)
	\end{align*}
	
	If in addition, ${\bf P}_{0|0} = {\bf 0}$, and ${\bf P}_{{\bf v}_t} = {\bf 0}$ then the the initial error covariance matrix does not change and $L^{\text{EKF}}_t (\boldsymbol{\theta}_t) = L^{\text{MLE}}_t (\boldsymbol{\theta}_t)$ for each $t$.
\end{proof}

\subsection{Proof of Theorem 2}
First let's define the distributions of interest. We adopt the notation from \cite{martens2014new}. Assume the inputs $u$ are drawn independently from a target distribution $Q_{u}$ with density function $q(u)$, and assume the corresponding outputs $y$ are drawn from a conditional target distribution $Q_{y|u}$ with density function $q(y|u)$. The target joint distribution is $Q_{u,y}$ whose density is $q(u,y) = q(y|u)q(u)$, and the learned distribution is $P_{u,y}(\boldsymbol{\theta})$, whose density is $p(u,y|\boldsymbol{\theta}) = p(y|u, \boldsymbol{\theta})q(u)$. 

\begin{lemma}
	If ${\bf P}_{{\bf n}_t}$ is diagonal with diagonal elements $\sigma_i=N$, then:
	\[\frac{1}{2}  \delta({\bf u}_t; \boldsymbol{\theta}_t) ^\top  {\bf P}_{{\bf n}_t}^{-1} \delta({\bf u}_t; \boldsymbol{\theta}_t) =  C +  N \mathbb{E}_{\hat{Q}_u} [D_{\text{KL}}(\hat{Q}_{y|u} || P_{y|u}(\boldsymbol{\theta}))]\]
\end{lemma}
\begin{proof}
	By definition:
	\[D_{\text{KL}}(Q_{u,y} || P_{u,y}(\boldsymbol{\theta})) = \int q(u, y) \log \frac{q(u, y)}{p(u,y|\boldsymbol{\theta})} dudy\]
	This is equivalent to the expected KL divergence over the conditional distributions.:
	\[\mathbb{E}_{Q_u} [D_{\text{KL}}(Q_{y|u} || P_{y|u}(\boldsymbol{\theta}))] \]
	since:
	{\small 
	\begin{align*}
	& \mathbb{E}_{Q_u} [D_{\text{KL}}(Q_{y|u} || P_{y|u}(\boldsymbol{\theta}))] = \int q(u) \int q(y|u) \log \frac{q(y|u)}{p(y|u,\boldsymbol{\theta})} dydu\\
	& = \int q(u, y) \log \frac{q(y|u)q(u)}{p(y|u,\boldsymbol{\theta})q(u)} dudy = D_{\text{KL}}(Q_{u,y} || P_{u,y}(\boldsymbol{\theta}))
	\end{align*}
	}
	Since we don't have access to $Q_u$ we substitute an empirical training distribution $\hat{Q}_u$ for $Q_u$ which is given by a set $\mathcal{S}_u$ of samples from $Q_u$. Then we define:
	\begin{align*}
	\mathbb{E}_{\hat{Q}_u} [D_{\text{KL}}(Q_{y|u} || P_{y|u}(\boldsymbol{\theta}))] & = \frac{1}{|\mathcal{S}|} \sum_{u \in \mathcal{S}_u} D_{\text{KL}}(Q_{y|u} || P_{y|u}(\boldsymbol{\theta}))
	\end{align*}
	In our training setting, we only have access to a single sample $y$ from $Q_{y|u}$ for each $u \in \mathcal{S}_u$, giving an empirical training distribution $\hat{Q}_{y|u}$. Then:
	{\small 
	\begin{align*}
	\mathbb{E}_{\hat{Q}_u} [D_{\text{KL}}(\hat{Q}_{y|u} || P_{y|u}(\boldsymbol{\theta}))] & = \frac{1}{|\mathcal{S}|} \sum_{(u, y) \in \mathcal{S}} 1 \log \frac{1}{p(y|u,\boldsymbol{\theta})}\\
	& = - \frac{1}{|\mathcal{S}|} \sum_{(u, y) \in \mathcal{S}}  \log p(y|u,\boldsymbol{\theta})
	\end{align*}
	}
	since $\hat{q}(y|u) = 1$.
	Now, back to our EKF notations. Assume that the $N$ observations in $y({\bf u}_t)$ are independent, then:
	\[\log p(y({\bf u}_t)|\boldsymbol{\theta}) = \log \Big(\prod_{i=1}^{N} p(y(u_t^i) | \boldsymbol{\theta}) \Big) = \sum_{i=1}^{N} \log p(y|u_t^i, \boldsymbol{\theta})\]
	where we changed the notation: $p(y(u_t^i) | \boldsymbol{\theta}) = p(y|u_t^i, \boldsymbol{\theta})$. Now let's write it explicitly for Gaussian distributions:
	{\small 
	\begin{align*}
	& \log p(y({\bf u}_t)|\boldsymbol{\theta}) = \log \bigg( \frac{1}{(2\pi)^{N/2} |{\bf P}_{{\bf n}_t}|^{1/2} } \\
	\nonumber & \quad \exp \Big( - \frac{1}{2} \big( y({\bf u}_t) - h({\bf u}_t; \boldsymbol{\theta}_t) \big)^\top  {\bf P}_{{\bf n}_t}^{-1} \big( y({\bf u}_t) - h({\bf u}_t; \boldsymbol{\theta}_t) \big) \Big) \bigg)\\
	& = C - \frac{1}{2} \big( y({\bf u}_t) - h({\bf u}_t; \boldsymbol{\theta}_t) \big)^\top  {\bf P}_{{\bf n}_t}^{-1} \big( y({\bf u}_t) - h({\bf u}_t; \boldsymbol{\theta}_t) \big)\\
	& = C - \frac{1}{2}  \delta({\bf u}_t; \boldsymbol{\theta}_t) ^\top  {\bf P}_{{\bf n}_t}^{-1} \delta({\bf u}_t; \boldsymbol{\theta}_t)
	\end{align*}
	} where $C = \log \big(\frac{1}{(2\pi)^{N/2} |{\bf P}_{{\bf n}_t}|^{1/2} } \big)$ is constant with respect to $\boldsymbol{\theta}$.
	Then we have that:
	\begin{align*}
	& \frac{1}{2}  \delta({\bf u}_t; \boldsymbol{\theta}_t) ^\top  {\bf P}_{{\bf n}_t}^{-1} \delta({\bf u}_t; \boldsymbol{\theta}_t) = C - \log p(y({\bf u}_t)|\boldsymbol{\theta})\\
	& = C -  \sum_{i=1}^{N} \log p(y|u_t^i, \boldsymbol{\theta}) = C +  N \mathbb{E}_{\hat{Q}_u} [D_{\text{KL}}(\hat{Q}_{y|u} || P_{y|u}(\boldsymbol{\theta}))]
	\end{align*}
	
	We have that:
	\[\small \boxed{\frac{1}{2}  \delta({\bf u}_t; \boldsymbol{\theta}_t) ^\top  {\bf P}_{{\bf n}_t}^{-1} \delta({\bf u}_t; \boldsymbol{\theta}_t) =  C +  N \mathbb{E}_{\hat{Q}_u} [D_{\text{KL}}(\hat{Q}_{y|u} || P_{y|u}(\boldsymbol{\theta}))]}\]
\end{proof}

\begin{lemma}
	For the empirical Fisher information matrix $\hat{F}$:
	\begin{align*}
	D_{\text{KL}} \big(P_{u,y}(\boldsymbol{\theta} + \Delta \boldsymbol{\theta})|| P_{u,y}(\boldsymbol{\theta}) \big) & = \frac{1}{2}  (\boldsymbol{\theta} - \boldsymbol{\hat{\theta}})^T {\bf \hat{F}} (\boldsymbol{\theta} - \boldsymbol{\hat{\theta}}) \\
	& + \mathcal{O}(\|\Delta \boldsymbol{\theta}\|^3)
	\end{align*}
\end{lemma}

\begin{proof}
	According to the KL-divergence definition:
	\begin{align*}
	& D_{\text{KL}} \big(P_{u,y}(\boldsymbol{\theta} + \Delta \boldsymbol{\theta})|| P_{u,y}(\boldsymbol{\theta}) \big) \\
	& =  \int p(u,y| \boldsymbol{\theta} + \Delta \boldsymbol{\theta}) \log p(u, y|\boldsymbol{\theta} + \Delta \boldsymbol{\theta}) dudy\\
	& - \int p(u, y| \boldsymbol{\theta} + \Delta \boldsymbol{\theta}) \log p(u,y| \boldsymbol{\theta}) dudy.
	\end{align*}
	
	According to Taylor expansion:
	\begin{align}
	\nonumber \log p(u,y|\boldsymbol{\theta}) &= \log p(u,y | \boldsymbol{\theta} + \Delta \boldsymbol{\theta} ) - {\bf g}^T \Delta \boldsymbol{\theta} + \frac{1}{2} \Delta \boldsymbol{\theta}^T {\bf H} \Delta \boldsymbol{\theta} \\
	& + \mathcal{O}(\|\Delta \boldsymbol{\theta}\|^3)
	\end{align} 
	where ${\bf g}$ is the gradient of $\log p(u,y|\boldsymbol{\theta})$ at the point $\boldsymbol{\theta} + \Delta \boldsymbol{\theta}$:
	\[{\bf g} = \nabla_{\boldsymbol{\theta}} \log p(u,y| \boldsymbol{\theta})_{|\boldsymbol{\theta} + \Delta \boldsymbol{\theta}}.\]
	Note that $p(u,y| \boldsymbol{\theta})=p(y| u, \boldsymbol{\theta} + \Delta \boldsymbol{\theta}) q(u)$. Since $q(u)$ does not depend on $\boldsymbol{\theta}$ then $\nabla_{\boldsymbol{\theta}} \log p(u,y| \boldsymbol{\theta}) = \nabla_{\boldsymbol{\theta}} \log p(y|u, \boldsymbol{\theta})$. Therefore, we can write ${\bf g}$ as:
	\begin{equation*}
		{\bf g} = \nabla_{\boldsymbol{\theta}} \log p(y| u, \boldsymbol{\theta})_{|\boldsymbol{\theta} + \Delta \boldsymbol{\theta}} = \begin{bmatrix}
		\frac{\partial \log p(y|u, \boldsymbol{\theta} + \Delta \boldsymbol{\theta})}{\partial \theta_1} \\ \vdots \\ \frac{\partial \log  p(y|u, \boldsymbol{\theta} + \Delta \boldsymbol{\theta})}{\partial \theta_d}
		\end{bmatrix}
	\end{equation*}
	Similarly, the Hessian ${\bf H}$ can be written as:
	\begin{align*}
		{\bf H} & = \nabla^2_{\boldsymbol{\theta}} \log p(u,y|\boldsymbol{\theta})_{|\boldsymbol{\theta} + \Delta \boldsymbol{\theta}} = \nabla^2_{\boldsymbol{\theta}} \log p(y|u, \boldsymbol{\theta})_{|\boldsymbol{\theta} + \Delta \boldsymbol{\theta}} \\
		& = \begin{bmatrix}
		\frac{\partial^2 \log p(y|u, \boldsymbol{\theta}+ \Delta \boldsymbol{\theta})}{\partial\theta_1^2} & \ldots & \frac{\partial^2 \log p(y|u, \boldsymbol{\theta} + \Delta \boldsymbol{\theta})}{\partial\theta_1 \partial\theta_d} \\ \vdots & \vdots & \vdots \\ \frac{\partial^2 \log p(y|u, \boldsymbol{\theta} + \Delta \boldsymbol{\theta})}{\partial\theta_d \partial\theta_1} & \ldots & \frac{\partial^2 \log p(y|u, \boldsymbol{\theta} + \Delta \boldsymbol{\theta})}{\partial \theta_d^2}
		\end{bmatrix}
	\end{align*}
	We use this Taylor expansion in the KL-divergence term, and use the notation: $\boldsymbol{\hat{\theta}} = \boldsymbol{\theta} + \Delta \boldsymbol{\theta} \quad \rightarrow \boldsymbol{\theta} -  \boldsymbol{\hat{\theta}} = - \Delta \boldsymbol{\theta}$.
	{\small 
	\begin{align*}
	& D_{\text{KL}} \big(P_{u,y}(\boldsymbol{\theta} + \Delta \boldsymbol{\theta})|| P_{u,y}(\boldsymbol{\theta}) \big) \\
	& =  \int p(u,y| \boldsymbol{\hat{\theta}}) \log p(u,y| \boldsymbol{\hat{\theta}}) dudy\\
	& - \int p(u,y| \boldsymbol{\hat{\theta}}) \Big( \log p(u,y|  \boldsymbol{\hat{\theta}} ) - {\bf g}^T \Delta \boldsymbol{\theta} + \frac{1}{2} \Delta \boldsymbol{\theta}^T {\bf H} \Delta \boldsymbol{\theta} \Big)  dudy \\
	& + \mathcal{O}(\|\Delta \boldsymbol{\theta}\|^3)  \\
	& =  \underbrace{\int p(u,y| \boldsymbol{\hat{\theta}}) \log p(u,y| \boldsymbol{\hat{\theta}}) dudy - \int p(u,y| \boldsymbol{\hat{\theta}}) \log p(u,y| \boldsymbol{\hat{\theta}} ) dudy }_{=0} \\
	& + \underbrace{\int p(u,y|\boldsymbol{\hat{\theta}}) \sum_{i=1}^{d} \frac{\partial \log p(y|u, \boldsymbol{\hat{\theta}})}{\partial \theta_i} \Delta \theta_i dudy}_{=0, see (*)} \\
	& \underbrace{- \frac{1}{2} \int  p(u,y| \boldsymbol{\hat{\theta}})  \sum_{i=1}^{d} \sum_{j=1}^{d} \Delta \theta_i \Delta \theta_j \frac{\partial^2 \log p(y|u,\boldsymbol{\hat{\theta}})}{\partial\theta_i \partial\theta_j}  dudy}_{=\frac{1}{2}  \Delta \boldsymbol{\theta}^T {\bf F} \Delta \boldsymbol{\theta}, see (**)} + \mathcal{O}(\|\Delta \boldsymbol{\theta}\|^3)\\
	& = \frac{1}{2}  \Delta \boldsymbol{\theta}^T {\bf F} \Delta \boldsymbol{\theta} + \mathcal{O}(\|\Delta \boldsymbol{\theta}\|^3)\\
	\end{align*}
	}
	We explain (*), according to regularities in the Leibniz integral rule (switching derivation and integral):
	{\small 
	\begin{align*}
	&\int p(u,y| \boldsymbol{\hat{\theta}}) \sum_{i=1}^{d} \frac{\partial \log p(y|u,\boldsymbol{\hat{\theta}})}{\partial \theta_i} \Delta \theta_i dudy \\
	&= \int q(u) \int p(y|u, \boldsymbol{\hat{\theta}}) \sum_{i=1}^{d} \frac{1}{p(y|u, \boldsymbol{\hat{\theta}})} \frac{\partial p(y|u, \boldsymbol{\hat{\theta}})}{\partial \theta_i} \Delta \theta_i dydu \\
	& = \int q(u) \sum_{i=1}^{d} \Delta \theta_i   \underbrace{\frac{\partial}{\partial \theta_i} \underbrace{\int p(y|u,\boldsymbol{\hat{\theta}}) dy}_{=1}}_{=0} du = 0
	\end{align*}
	}
	We explain (**):
	{\small 
	\begin{align*}
	& - \frac{1}{2} \int  p(u,y| \boldsymbol{\hat{\theta}})  \sum_{i=1}^{d} \sum_{j=1}^{d} \Delta \theta_i \Delta \theta_j \frac{\partial^2 \log p(y|u, \boldsymbol{\hat{\theta}})}{\partial\theta_i \partial\theta_j} du dy \\
	& = - \frac{1}{2} \int q(u) \sum_{i=1}^{d} \sum_{j=1}^{d} \Delta \theta_i \Delta \theta_j \cdot \\
	& \int  p(y|u, \boldsymbol{\hat{\theta}})   \frac{\partial }{\partial\theta_i} \Big(\frac{1}{ p(y|u, \boldsymbol{\hat{\theta}})}\frac{\partial p(y|u, \boldsymbol{\hat{\theta}})}{ \partial\theta_j} \Big)  dy du\\
	& = - \frac{1}{2} \int q(u) \sum_{i=1}^{d} \sum_{j=1}^{d} \Delta \theta_i \Delta \theta_j \int  p(y|u, \boldsymbol{\hat{\theta}})  \\
	& \Big( \frac{1}{ p(y|u, \boldsymbol{\hat{\theta}})} \frac{\partial^2 p(y|u, \boldsymbol{\hat{\theta}})}{\partial\theta_i \partial\theta_j} \\
	& - \frac{1}{ p(y|u, \boldsymbol{\hat{\theta}})^2} \frac{\partial p(y|u, \boldsymbol{\hat{\theta}}) }{\partial\theta_i} \frac{\partial p(y|u, \boldsymbol{\hat{\theta}})}{ \partial\theta_j}   \Big) dy du\\
	& = - \frac{1}{2} \int q(u) \sum_{i=1}^{d} \sum_{j=1}^{d} \Delta \theta_i \Delta \theta_j \int  \Big( \frac{\partial^2 p(y|u, \boldsymbol{\hat{\theta}})}{\partial\theta_i \partial\theta_j} \\
	&- p(y|u, \boldsymbol{\hat{\theta}}) \frac{\partial \log p(y|u, \boldsymbol{\hat{\theta}}) }{\partial\theta_i} \frac{\partial \log p(y|u, \boldsymbol{\hat{\theta}})}{ \partial\theta_j} \Big)  dy du\\
	& = - \frac{1}{2} \int q(u) \sum_{i=1}^{d} \sum_{j=1}^{d} \Delta \theta_i \Delta \theta_j \underbrace{\frac{\partial^2}{\partial\theta_i \partial\theta_j} \underbrace{\int p(y|u, \boldsymbol{\hat{\theta}}) dy}_{=1}}_{=0} du \\
	& + \frac{1}{2} \int q(u) \sum_{i=1}^{d} \sum_{j=1}^{d} \Delta \theta_i \Delta \theta_j\\
	& \quad \quad \mathbb{E}_{P_{y|u}(\boldsymbol{\hat{\theta}})} \Big[  \frac{\partial \log p(y|u, \boldsymbol{\hat{\theta}}) }{\partial\theta_i} \frac{\partial \log p(y|u, \boldsymbol{\hat{\theta}})}{ \partial\theta_j} \Big] du\\
	& = \frac{1}{2}  \Delta \boldsymbol{\theta}^T {\bf F} \Delta \boldsymbol{\theta}
	\end{align*}
	}
	where
	\[{\bf F}_{ij} = \mathbb{E}_{Q_{u}} \Bigg[  \mathbb{E}_{P_{y|u}(\boldsymbol{\hat{\theta}})} \Big[  \frac{\partial \log p(y|u, \boldsymbol{\hat{\theta}}) }{\partial\theta_i} \frac{\partial \log p(y|u, \boldsymbol{\hat{\theta}})}{ \partial\theta_j} \Big] \Bigg]\]
	Since we don't have access to $Q_u$ we will use the empirical training distribution $\hat{Q}_u$:
	\[{\bf \hat{F}}_{ij} = \frac{1}{|\mathcal{S}|} \sum_{u \in \mathcal{S}_u}   \mathbb{E}_{P_{y|u}(\boldsymbol{\hat{\theta}})} \Big[  \frac{\partial \log p(y|u, \boldsymbol{\hat{\theta}}) }{\partial\theta_i} \frac{\partial \log p(y|u, \boldsymbol{\hat{\theta}})}{ \partial\theta_j} \Big] \]
	We received that:
	\begin{align*}
	D_{\text{KL}} \big(P_{u,y}(\boldsymbol{\theta} + \Delta \boldsymbol{\theta})|| P_{u,y}(\boldsymbol{\theta}) \big) = \\
	 \frac{1}{2}  (\boldsymbol{\theta} - \boldsymbol{\hat{\theta}})^T {\bf \hat{F}} (\boldsymbol{\theta} - \boldsymbol{\hat{\theta}}) + \mathcal{O}(\|\Delta \boldsymbol{\theta}\|^3)
	\end{align*}
\end{proof}
Now we can summarize the proof for Theorem 2: 
\begin{proof}
	Adding the relationship from \cite{ollivier2018online}: ${\bf \hat{F}}_{t|t-1} = \frac{1}{t} {\bf P}_{t|t-1}^{-1}$, and combining the results from Lemma 1 and Lemma 2, our objective function can be approximated as:
	{\small 
	\begin{align*}
	 L^{\text{EKF}}_t(\boldsymbol{\theta}_t)	& = \frac{1}{2}  \delta({\bf u}_t; \boldsymbol{\theta}_t)^\top {\bf P}_{{\bf n}_t}^{-1}  \delta({\bf u}_t; \boldsymbol{\theta}_t) \\
	& \quad +  \frac{1}{2}(\boldsymbol{\theta}_t - \boldsymbol{\hat{\theta}}_{t|t-1})^\top {\bf P}_{t|t-1}^{-1} (\boldsymbol{\theta}_t - \boldsymbol{\hat{\theta}}_{t|t-1}) \\
	& = C +  N \mathbb{E}_{\hat{Q}_u} [D_{\text{KL}}(\hat{Q}_{y|u} || P_{y|u}(\boldsymbol{\theta}))] \\
	& \quad +  \frac{t}{2}(\boldsymbol{\theta}_t - \boldsymbol{\hat{\theta}}_{t|t-1})^\top {\bf \hat{F}}_{t|t-1} (\boldsymbol{\theta}_t - \boldsymbol{\hat{\theta}}_{t|t-1})\\
	& \approx C +  N \mathbb{E}_{\hat{Q}_u} [D_{\text{KL}}(\hat{Q}_{y|u} || P_{y|u}(\boldsymbol{\theta}))] \\
	& \quad +  t \cdot D_{\text{KL}} \big(P_{u,y}(\boldsymbol{\theta} + \Delta \boldsymbol{\theta})|| P_{u,y}(\boldsymbol{\theta}) \big) \\
	\end{align*}
	}
	which completes the proof.
\end{proof}

\section{Experimental details} 
Our experiments are based on the baselines implementation  \cite{baselines} for PPO and TRPO. We used their default hyper parameters, and only changed the optimizer for the value function from Adam to KOVA. For brevity, we bring here the network architecture and the hyper parameters for each algorithm. 

{\bf PPO:}
Following \citep{schulman2017proximal}, the policy network is a fully-connected MLP with two hidden layers, 64 units and tanh nonlinearities. The output of the policy network is the mean and standard deviations of a Gaussian distribution of actions for a given (input) state.
The value network is a fully-connected MLP with two hidden layers, 64 units and tanh nonlinearities. The output of the value network is a scalar, representing the value function for a given (input) state. PPO uses the GAE estimator for the advantage function \cite{schulman2015high}. 
In Tables \ref{PPO_table} and \ref{Kalman_for_PPO_table} we present the hyper parameters for the PPO experiments. The Horizon represents the number of timesteps per each policy rollout.

\begin{table}[htp]
	\parbox[t]{1.\linewidth}{
		\centering
		\caption{PPO hyper-parameters used for the Mujoco tasks}
		\begin{tabular}{|c|c|}
			\hline
			{\bf Hyper-parameter} & {\bf Value} \\ \hline \hline
			Horizon &  2048 \\ \hline 
			Adam learning rate & $3 \cdot 10^{-4}$  \\ \hline
			Num. epochs & 10 \\ \hline
			Minibatch size & 64 \\ \hline
			Discount $(\gamma)$ & 0.99 \\ \hline
			GAE parameter $(\lambda)$ & 0.95 \\ \hline
			Clip range & 0.2 \\ \hline
		\end{tabular}
		\label{PPO_table}
	} \quad \quad
	\parbox[t]{1.1\linewidth}{
	\centering
	\caption{KOVA hyper-parameters used for VF optimization in PPO}
	\begin{tabular}{|l|l|}
		\hline
		{\bf Hyper-parameter} & {\bf Value} \\ \hline \hline
		KOVA learning rate & $1.0$ (Swimmer, HalfCheetah,   \\ 
		& \quad \ \ \ Walker2d) \\ 
		& $0.1$ (Hopper, Ant) \\ \hline
		${\bf P}_{{\bf n}_t}$ type & max-ratio \\ \hline
		$\eta$ &  $0.1$ (Hopper, HalfCheetah, Ant) \\ 
		& $0.01$ (Swimmer, Walker2d)  \\ \hline
	\end{tabular}
	\label{Kalman_for_PPO_table}
	}
\end{table} 

{\bf TRPO:}
The policy network and the value network are the same as described for PPO, only with 32 units instead of 64. TRPO also uses the GAE estimator. In Tables \ref{TRPO_table} and \ref{Kalman_for_TRPO_table} we present the hyper parameters for the TRPO experiments.

\begin{table}[htp]
	\parbox[t]{1.\linewidth}{
	\centering
	\caption{TRPO hyper-parameters used for Mujoco tasks}
	\begin{tabular}{|c|c|}
		\hline
		{\bf Hyper-parameter} & {\bf Value} \\ \hline \hline
		Horizon &  1024 \\ \hline 
		Batch size & 64 \\ \hline
		Discount $(\gamma)$ & 0.99 \\ \hline
		GAE parameter $(\lambda)$ & 0.98 \\ \hline
		Max KL & 0.01 \\ \hline
		Conjugate gradient iterations & 10 \\ \hline
		Conjugate gradient damping & 0.1 \\ \hline
		VF iterations & 5 \\ \hline
		VF learning rate & $10^{-3}$ \\ \hline
		Normalize observations & True \\ \hline
	\end{tabular}
	\label{TRPO_table}
	}
	\quad \quad 
	\parbox[t]{1.1\linewidth}{
	\centering
	\caption{KOVA hyper-parameters used for VF optimization in TRPO}
	\begin{tabular}{|l|l|}
		\hline
		{\bf Hyper-parameter} & {\bf Value} \\ \hline \hline
		KOVA learning rate & $1.0$ (Swimmer, Hopper)  \\ 
		& $0.1$ (HalfCheetah) \\ 
		& $0.01$ (Ant, Walker2d) \\ \hline
		${\bf P}_{{\bf n}_t}$ type & max-ratio \\ \hline
		$\eta$ & $0.01$ \\ \hline
	\end{tabular}
	\label{Kalman_for_TRPO_table}
	}
\end{table} 

\end{document}